\documentclass{article}
\usepackage{times}
\usepackage{amsmath,amssymb,amsthm,graphicx}
\usepackage[mathscr]{eucal}
\usepackage{cite}
\usepackage{bbm}
\usepackage{algorithm,algorithmic}
\usepackage{hyperref}
\usepackage{cleveref}
\usepackage{tcolorbox}
\usepackage{enumerate}
\usepackage{mathtools}
\usepackage{natbib}
\usepackage{fullpage}
\usepackage{color}

\usepackage{tikz}
\usetikzlibrary{decorations.pathreplacing}

\usepackage{fullpage}
\DeclarePairedDelimiter\paren{\lparen}{\rparen}
\DeclarePairedDelimiter\abs{|}{|}

\newcommand{\ignore}[1]{}

\newcommand{\F}{{\cal F}}
\newcommand{\G}{{\cal G}}
\newcommand{\Q}{{\cal Q}}
\newcommand{\X}{{\cal X}}

\newcommand{\R}{{\mathbb R}}

\newcommand{\tv}{\mathsf{TV}}
\newcommand{\N}{\mathbb{N}}
\newcommand{\D}{\mathcal{D}}
\newcommand{\Ex}{\mathbb{E}}

\newcommand{\eps}{\epsilon}
\newcommand{\opt}{\mathsf{opt}}
\newcommand{\conv}{\mathsf{conv}}
\newcommand{\poly}{\mathsf{poly}}

\newtheorem{theorem}{Theorem}
\newtheorem{lemma}[theorem]{Lemma}
\newtheorem{claim}[theorem]{Claim}

\newtheorem{definition}[theorem]{Definition}
\newtheorem{corollary}[theorem]{Corollary}

\newtheorem{remark}{Remark}

\title{The Optimal Approximation Factor in Density Estimation}
\author{
Olivier Bousquet\thanks{Google Brain, Z{\"u}rich. {\tt obousquet@google.com.}}
\and  Daniel M. Kane\thanks{Department of Computer Science and Engineering/Department of Mathematics, University of California, San Diego. {\tt dakane@ucsd.edu.}}
\and Shay Moran\thanks{Department of Computer Science, Princeton University. {\tt  shaymoran1@gmail.com.}}
}

\begin{document}

\maketitle

\begin{abstract}
Consider the following problem: given two arbitrary densities $q_1,q_2$
	and a sample-access to an unknown target density $p$,
	find which of the $q_i$'s is closer to $p$ in total variation.

A remarkable result due to Yatracos shows that this problem is tractable in the following sense:
	there exists an algorithm that uses  $O(\epsilon^{-2})$ samples from $p$ and outputs~$q_i$
	such that with high probability, $TV(q_i,p) \leq 3\cdot\opt + \epsilon$,
	where $\opt= \min\{TV(q_1,p),TV(q_2,p)\}$.
Moreover, this result extends to any finite class of densities $\Q$:
	there exists an algorithm that outputs the best density in $\Q$
	up to a multiplicative approximation factor of 3.

We complement and extend this result by showing that:
	(i) the factor 3 can not be improved if one restricts the algorithm to output a density from $\Q$, and
	(ii) if one allows the algorithm to output arbitrary densities (e.g.\ a mixture of densities from $\Q$), 
	then the approximation factor can be reduced to 2, which is optimal.
         In particular this demonstrates an advantage of improper learning over proper in this setup.

We develop two approaches to achieve the optimal approximation factor of $2$: an adaptive one and a static one.
	Both approaches are based on a geometric point of view of the problem
	and rely on estimating surrogate metrics to the total variation.
	Our sample complexity bounds exploit techniques from {\it Adaptive Data Analysis}.

\end{abstract}

\section{Introduction}

We study the problem of agnostic distribution learning whereby a learner is given i.i.d.\ samples from an {\em unknown} distribution $p$  and needs to choose, among a set $\Q$ of candidate distributions, the one that is closest to $p$.
This problem formulation immediately raises several questions. The first one is how to define close-ness between probability distributions. Here we will argue that the total variation metric is a natural choice. The second one is what assumptions are made on $p$.
We choose the so called agnostic or robust case which means that we are not making any assumption.
The last one is whether the best thing to do for the learner is to return an element of $\Q$ (this is called the proper case), or to possibly produce a distribution which is not a member of $\Q$ (this is the improper case) but is guaranteed to be competitive with respect to the best member of~$\Q$.

Our study will focus on the information-theoretic limits of the problem, which means that we will not be concerned with the computational complexity of the learner and will only consider what, in theory, is the best achievable performance of a learner as a function of the size of the candidate class $\Q$ and the number $m$ of samples from $p$ that it has access to.

%\subsection{Distribution learning as a fundamental learning problem}
%Relation to model selection (the finite case)
%ordinary problem in science // we really consider a few candidates over a large domain (possibly infinite) // other %situation: infinite class and then epsilon-net. In the first case, ok to be linear in the size of the class, in the second %case you want to be logarithmic in the size of the class // we have results for both (but possibly very far from optimal %-- and computational complexity is really bad for the adaptive (O(X)))

\subsection{Why Total Variation?}
The total variation metric, defined for two probability measures $p,q$ on $\X$ as
\begin{equation}\label{eq:tv}
TV(p,q):=\sup_{A\subset \X}\left|p(A)-q(A)\right|\,,
\end{equation}
has the nice property of being a proper metric. Additionally it has the natural interpretation of measuring the largest discrepancy in the measure assigned to the same event by the two different measures. And while it thus looks like an $L_\infty$ metric (when viewing a probability measure as a map from subsets of $\X$ to $[0,1]$),
it also can be rewritten as an $L_1$ norm: if $p$ and $q$ have densities
$dp$ and $dq$ respectively (or probability mass function when $\X$ is finite/countable),
\begin{equation}\label{eq:tv}
TV(p,q)=\frac{1}{2}\left\|dp-dq\right\|_1\,,
\end{equation}
as well as an optimal coupling:
\begin{equation}\label{eq:tv}
TV(p,q)=\inf_{(Y,Z):Y\sim p, Z\sim q} \mathbb{P}(Y\neq Z)\,.
\end{equation}

Note that there is a large literature about density estimation in the $L_2$ metric (as opposed to $L_1$). However, $L_2$ is a less natural way of measuring the distance between densities because it lacks invariance with respect to the choice of the reference measure on the domain. This may not be an issue when considering real-valued distributions where the Lebesgue measure is the canonical choice, but when working on high-dimensional or general domains, this dependency is not necesssarily desirable (for more details, see Chapter 6.5 in the book by~\cite{Devroye01combinatorial}).

Another classical choice is to use the Kullback-Leibler divergence, however $KL(q,p)$ has the down-side of being defined only when $q$ is absolutely continuous with respect to $p$ and in a setting like the one we are considering where we do not wish to assume anything about the target distribution, this cannot be guaranteed. Even if one were to consider $KL(p,q)$ instead, then one would be restricted to considering models that put mass on all points of the domain and the Kullback-Leibler distance could be dominated by the points of very low $q$ probability.

Compared to those other two choices, total variation has the benefit of being invariant, bounded and being a metric.
We refer the reader to Chapter 6 in the book by~\cite{Devroye01combinatorial}
for a discussion regarding the advantages of total variation and a detailed comparison with other natural similarity measures.

Of course, there are other possible choices such as the Hellinger divergence or others, and it would be an interesting question to extend the current study to those.

\subsection{Why Agnostic?}
A basic classification of machine learning problems separates between {\it realizable} and {\it agnostic} learning.
In the realizable case one assumes that the target distribution $p$ belongs to a prespecified class $\Q$ which is known to the algorithm,
and in the agnostic case one usually does not assume anything about the target distribution $p$ but rather extends
the goal of learning to so that the output distribution $q$ is competitive with the best distribution in $\Q$ (i.e.\ the one which is closest to $p$).

In this work we focus on the agnostic case.
Nevertheless,  a sensible\footnote{This is due to the lower bound of $2\opt$ (and $3\opt$ in the proper case),
see \cref{sec:mainresults}.} setting to keep in mind is the ``almost realizable''
case in which the distance between $p$ and $\Q$ is small.
Such scenarios may occur in contexts where one has a strong prior about the target distribution,
but would like to remain resilient/robust against small fluctuations and thus to avoid realizability assumptions.
%A basic classification of machine learning problems
%Define agnostic versus realizable learnability.
%We don't necessarily want to make assumptions on the true data generation mechanism.
%Sometimes the context suggests a natural class of distributions, but perhaps
%the target distribution does not belong to the class and only close to some distribution in it.
%Therefore, it is sensible to avoid properness assumptions.

\subsection{Why Improper?}
Another basic classification in machine learning problems distinguishes between {\it proper} and {\it improper}
learning. In the proper case the algorithm always outputs a distribution $q\in \Q$ whereas in the improper
case it may output arbitrary distribution (in both cases the goal remains the same, namely to compete with the
best distribution in $\Q$). While at a first glance it may seem strange to consider the improper case,
it turns out that in many cases improperness is beneficial (e.g.\ boosting is inherently improper~\citep{schapire2012boosting}; in multiclass classification some classes can only be learned improperly~\citep{Daniely14proper}). The main results in this paper manifest another setting in which improper learning is provably stronger than proper learning.

\subsection{Is this problem too hard?}
While the total variation is a natural metric with strong guarantees,
	at a first glance it may seem impossible to use in such an abstract distribution learning setting:
	imagine that the class $\Q$ contains just two distributions $q_1,q_2$, and let $p$ denote the target distribution.
	Then, a natural empirical-risk-minimization-like approach would be to estimate both distances $\tv(q_1,p),\tv(q_2,p)$
	from a large enough i.i.d.\ sample drawn from $p$ and output the minimizer.
	The problem with this approach is that estimating $\tv(\cdot,p)$
	requires $\Omega(\lvert \X\rvert)$ samples from $p$ (see e.g.~\cite{Jiao18minimax}).
	In particular, if $\X$ is infinite (say~$\X=\mathbb{R}$) then it is impossible to do it with a finite sample complexity.

However, perhaps surprisingly, despite the impossibility of estimating the total variation
	one can still find an approximate minimizer of it (even when $\X$ is infinite!).
	 A more detailed survey of relevant results is given in \Cref{sec:prevwork} below.

\subsection{Problem Definition}
Let $\X$ be a domain and let $\Delta(\X)$ denote the set of all probability distributions over $\X$.
We assume that either (i) $\X$ is finite in which case $\Delta(\X)$
is identified with the set of $\lvert \X\rvert$-dimensional probability vectors, or
(ii) $\X=\R^d$ in which case $\Delta(\X)$ is the set of Borel probability measures.

%To ease presentation and avoid measure theoretic subtleties
%we assume that $X$ is countable (i.e.\ that every subset of $X$ is measurable).
%We remark however that the statements
%in this paper apply also to continuous domain
%under standard measurability assumption.

Let $\Q\subseteq \Delta(\X)$ be a set of distributions.
We focus on the case where $\Q$ is finite and denote its size by~$n$.
Let~$\alpha > 0$, we say that $\Q$ is {\it $\alpha$-learnable} if there is a {(possibly randomized)} algorithm $A$ such that
for every $\eps,\delta>0$ there is a finite sample complexity bound $m=m(\eps,\delta)$
such that for every target distribution $p\in \Delta(\X)$,
if $A$ receives as input at least $m$ independent samples from $p$
then it outputs a distribution $q$ such that
\[\tv(p,q)\leq \alpha\cdot\opt + \eps,\]
with probability at least $1-\delta$, where $\opt = \min_{q\in \Q}\tv(p,q)$
and $\tv(p,q) = \sup_{A\subseteq \X}\{p(A)-q(A)\}$ is the total variation distance.
We say that $\Q$ is {\it properly $\alpha$-learnable} if it is $\alpha$-learnable
by a proper algorithm; namely an algorithm that always outputs~$q\in \Q$.
The function $m=m(\eps,\delta)$ is called the sample complexity of the algorithm.

\paragraph{Sample complexity.}
	Note that if $\X$ is finite then any class of distribution is $\alpha$-learnable for $\alpha=1$
	with sample complexity $O(\lvert \X\rvert /\eps^2)$
	(because this many samples suffice to estimate $p(A)$ for every $A\subseteq \X$,
	which allows to estimate its total variation distance to each $q\in \Q$).
	{\it Therefore, when $\X$ is finite, we consider $\Q$ to be $\alpha$-learnable only if its sample complexity
	depends efficiently on $\lvert \X\rvert$, namely $\poly\log(\lvert \X\rvert)$}
	(note that $\log\lvert \X\rvert$ is the bit-complexity of representing each sample
	in the input and therefore $\poly\log\lvert \X\rvert$ means polynomial in the input size).

\subsection{Previous Related Work}\label{sec:prevwork}
Density estimation has been studied since more than a century ago,
for textbook introductions see e.g.~\citep{devroye85nonparametric,Devroye01combinatorial,Diakonikolas16survey}.
A significant portion of works considered this problem when $\Q$ is some specific class of distributions
such as mixtures of gaussians (e.g.~\cite{Kalai12mixtures,Diakonikolas17mixtureslower,Diakonikolas18mixtures,Kothari18robust,Ashtiani18mixtures1,Ashtiani18mixtures}), histograms (e.g.~\cite{Pearson95contributions,lugosi96histograms,Devroye04bin,Chan14histograms,Diakonikolas18histograms}), and more. For a fairly recent survey see~\citep{Diakonikolas16survey}.

This work concerns arbitrary classes $\Q$ and the only assumption we make is that $\Q$ is finite.
The factor 3 upper bound in the proper case was derived by \cite{yatracos85}
using the elegant and simple idea of Yatracos' sets (also referred to as Schaffe's sets by~\cite{Devroye01combinatorial}).
\cite{Devroye01combinatorial} extended Yatracos' idea and also gave a factor 2 lower bound for his algorithm.
\cite{Mahalanabis08density} improved the lower bound to 3 and extended it to a more general family of proper algorithms. 
\cite{Mahalanabis08density} also showed that in the case of $n=2$ distributions, the exists a {\it randomized proper} algorithm,
which achieves a factor $2$. approximation
A lower bound of factor 2 for arbitrary (possibly improper) algorithms follows from the work~\cite{Chan14histograms} (see \cref{sec:mainresults}).
\cite{Devroye01combinatorial} point out in their book
the absence of universal methods other than Yatracos' which achieve
a constant approximation factor; this comment inspired the current work.

\subsection{Main Results}\label{sec:mainresults}

\begin{theorem}[Upper bound - improper case]\label{thm:improperub}
Every finite class of distributions $\Q$ is $\alpha$-learnable with $\alpha=2$.
%Let $Q$ be a class of distributions over $X$ with $\lvert Q\rvert = n$.
%Then $Q$ is $\alpha$-learnable with $\alpha=2$
%with sample complexity $m \leq \min(m_1,m_2,m_3)$, where
%\begin{align*}
%m_1(\eps,\delta) &= O\bigl(\frac{n}{\eps^2}\log(1/\delta)\bigl), \\
%m_2(\eps,\delta) &= O\bigl(\frac{\sqrt{n}}{\eps^{5/2}}\log(1/\delta)\bigr),\\
%m_3(\eps,\delta) &= O\Bigl(\frac{\sqrt{\log\lvert X\rvert}\log n}{\eps^{7/2}}\log(1/\delta)\Bigr) .
%\end{align*}
\end{theorem}

We prove \Cref{thm:improperub} and provide explicit sample complexity bounds in \Cref{sec:ub}.

\begin{theorem}[Lower bound - proper case]\label{thm:properlb}
For every $\alpha < 3$ there is a class $\Q$ of size $2$ that is not properly $\alpha$-learnable.
\end{theorem}
We prove \Cref{thm:properlb} and provide explicit sample complexity bounds in \Cref{sec:lb}.

\begin{remark}
{A recent follow-up work by~\cite*{aamand2025fastgoodhypothesis}, which builds on and extends our lower
bound, identified a subtle gap in the construction used in our original
proof.  We are grateful to the authors for bringing this to our attention.
The proof in the present manuscript incorporates the necessary corrections.
The construction of the two key distributions, as well as the overall proof
strategy, remain essentially the same; only certain technical details in the
argument have been refined.}
\end{remark}

\paragraph{Tightness of \Cref{thm:improperub}.}
The factor $\alpha=2$ in \Cref{thm:improperub} in general can not be improved.
	This follows from~\cite{Chan14histograms} (Theorem 7) which demonstrates a class $\Q$ of distributions over $\{1,\ldots,N\}$
	such that any (possibly improper) algorithm that $\alpha$-learns
	this class with $\alpha < 2$ requires some $\Omega(\sqrt{N})$ samples.
	Note that in their Theorem statement the class $\Q$ is infinite,
	but a closer inspection of their proof reveals that it needs only to contain
	two distributions, and so their lower bound already applies for $\lvert \Q\rvert = 2$.

\paragraph{Proofs overview.}
Our approach for the lower bound is a variant of the proof in~\cite{Chan14histograms} and boils down to using a tensorized version of Le Cam's method together with a birthday paradox kind of argument.

For the upper bound, we introduce two methods, a static and an adaptive one, both of which are based on the observation that once we find a distribution $q$ so that $\tv(q,q_i)\leq \tv(p,q_i)+\epsilon$ for every $q_i\in \Q$ the result follows by the triangle inequality (see Lemma \ref{lem:termination}). The static method can be viewed as a direct extension of Yatracos' ideas as we also construct a family of functions of finite VC dimension and estimate the corresponding surrogate variational metric (see Equation~\eqref{eq:df}). Note however that our construction and analysis are more complex and rely on a careful inspection of barycenters with respect to the total variation metric.

The adaptive method, which could apply to other probability metrics\footnote{As long as they have a variational form as in \eqref{eq:df}, which is for example the case of Wasserstein's metric.} than $\tv$ proceeds in steps: it maintains lower bounds $z_i\leq \tv(p,q_i)$ and, at each step, increases one of them by at least $\epsilon$ until there exists a distribution $q$ such that $\tv(q,q_i)\le z_i+\epsilon$ for all $i$. Given that $\tv$ is bounded by $1$, this implies that the algorithm terminates after $|\Q|/\epsilon$ steps. The crux of the algorithm is in the implementation of each step. To this end we use the minimax theorem applied to $\min_q \tv(q,q_i)$ (since $\tv$ is a supremum) to find  functions $f_i$ so that some linear combination of the numbers $|\Ex[f_i(q)]-\Ex[f_i(q_i)]|-z_i-\epsilon$ is positive for \emph{any} distribution $q$. Applying this result for $q=p$ implies that estimating $\Ex[f_i(p)]$ will allow us to improve at least one of our lower bounds.

\subsection{Open Questions and Future Research}
The main result in this paper is the determination of the optimal approximation
factor in density estimation and the development of universal algorithmic approaches to achieve it.

One central issue that remains open concerns sample complexity.
Our current sample complexity upper bounds are either linear in $\lvert \Q\rvert$
or based on rather sophisticated techniques from adaptive data analysis
which includes dependencies on $\log\lvert \X\rvert$.
For comparison, Yatracos' proper algorithm which achieves factor 3
has a clean sample complexity of $\frac{\log\lvert \Q\rvert}{\eps^2}$.
It would be interesting to determine whether the factor $2$ can be achieved with a similar sample complexity.

We list below other possible suggestions for future research:
\begin{itemize}
\item {\cite{Mahalanabis08density} consider the case of $\Q=\{q_1,q_2\}$ 
and provide a {\it randomized} proper algorithm which outputs $q_i\in \Q$
such that $\Ex[\tv(q_i,p)] \leq 2\opt +o(1)$ (see Theorem 10 in \citep{Mahalanabis08density}).
Can this result be extended to arbitrary finite $\Q$?}
%\item Improve the sample complexity of $\alpha$-learning for $\alpha=2$.
\item Is it the case that any (possibly infinite) class $\Q$ that is $\alpha$-learnable for some $\alpha$
is $\alpha$-learnable for $\alpha=2$? E.g.\ assume that the family of Yatracos' sets of $\Q$
has a finite VC dimension (so $\Q$ is properly $\alpha$-learnable for $\alpha=3$).
Is $\Q$ $\alpha$-learnable for $\alpha=2$?
%\item Do there exists a randomized proper algorithm which achieves a factor $2$?
%Relatedly, can our algorithm always output a distribution which is a mixture of the $q_i$'s?
%\item What is the sample complexity of $\alpha$-learning for $\alpha\in (2,3)$?
%We know that for $\alpha=3$ it is $O(\log n)$ and we have various inferior bounds for $\alpha=2$.
%What can be said e.g.\ for $\alpha=2.5$?
\item Our result remains valid if we replace the total variation with any IPM\footnote{I.e.\ any metric defined by $d(p,q) = \sup_{f\in \F}\lvert \Ex_p[f]-\Ex_q[f]\rvert$, where $\F$ is a family of $\X\to[0,1]$ functions.} metric.
How about $f$-divergences? Is there a natural characterization of all $f$-divergences
for which every finite $\Q$ can be $\alpha$-learned for some constant $\alpha<\infty$?
%\item Is the solution always a mixture? (every vertex of the polytope can be realized by a distribution in the convex hull of \Q) (in this case there is a randomized proper algorithm).
\end{itemize}

\section{Preliminaries}

\paragraph{An assumption.}
Some of our arguments exploit the Minimax Theorem for zero-sum games~\citep{Neumann1928}.
	Therefore, we will assume a setting (i.e.\ the domain $\X$ and the set of distributions $\Q\subseteq \Delta(\X)$)
	in which this theorem is valid. Alternatively, one could state explicit assumptions
	such as finiteness or forms of compactness under which it is known that the Minimax Theorem holds.
	However, we believe that the presentation benefits from avoiding such explicit technical assumptions
	and simply assuming the Minimax Theorem as an ``axiom'' in the discussed setting.

\paragraph{Standard notation.}
We use $[N]$ to denote the set $\{1,\ldots,N\}$.
	For two vectors $u,v\in\R^n$ let $u\leq v$ denote the statement that $u_i\leq v_i$ for every $i$.
	Denote by $e_i$ the standard basis vector whose $i$'th coordinate is $1$ and its other coordinates are $0$
	and by $1_n$ the vector~$(1,\ldots,1)\in\R^n$.

We use standard notations for asymptotics such as $O,o,\Omega,\omega,\Theta$.
We may also sometimes use~$\tilde O$ or $\tilde \Omega$ to hide logarithmic factors.
E.g.\ $f=\tilde O(g)$ if $f = O(g \log^c (g))$ for some $c\in\N$.

%\subsection{Distribution learning under total-variation}

\subsection{Total Variation and Surrogates}

Let $\F$ be a family of $\X\to[0,1]$ functions.
	Assume that $\F$ is symmetric in the sense that whenever $f\in \F$ then also $1-f\in \F$
	(this allows us to remove the absolute value from some definitions and will simplify some calculations).
	Define a semi-metric on $\Delta(\X)$ (recall that $\Delta(\X)$ is the set of distributions over $\X$),
	\begin{equation}\label{eq:df}
	d_\F(p,q) = \sup_{f\in F}\bigl\{\Ex_{x\sim p}[f(x)]-\Ex_{x\sim q}[f(x)]\bigr\}.
	\end{equation}
	Note that when $\F$ is the set of all (measurable) $\X\to [0,1]$ functions then $d_\F$ is the total variation distance,
	that~$d_\F(p,q)$ is symmetric, i.e.\ $d_\F(p,q)=d_\F(q,p)$, and that and that $d_\F(\cdot,q):\Delta(\X)\to \R$ is convex (as a supremum over linear functions).

\paragraph{Distances vectors and sets.}

Let $\Q=\{q_1,\ldots,q_n\}\subseteq \Delta(\X)$,
and let $p$ be a distribution.
The $\F$-distance vector of $p$ relative to the $q_i$'s is the vector~$v=v(p)=(d_\F(p,q_i))_{i=1}^n$.

The following claim shows that in order to find $q$ such that $d_{\F}(q,p)\leq 2\min_{i}d_{\F}(q_i,p)+\eps$
	it suffices to find $q$ such that $v(q)\leq v(p)+\eps\cdot1_n$.
	All of our algorithms exploit this claim.
\begin{lemma}\label{lem:termination}
Let $q,p$ such that $v(q)\leq v(p) + \eps1_n$. Then $\tv(q,p)\leq 2\min_{i} \tv(q_i,p) + \eps$.
\end{lemma}
\begin{proof}
Follows directly by the triangle inequality; indeed, let $q_i$ be a minimizer of $\tv(\cdot,p)$ in $\Q$. Then,
$\tv(q,p)\leq \tv(q,q_i) + \tv(q_i,p) \leq (\tv(p,q_i) + \eps) + \tv(q_i,p) = 2\tv(q_i,p) + \eps$.
\end{proof}

Next, we explore which $v\in\R^n$ are of the form $v=v(p)$ for some $p\in\Delta(\X)$. 
For this we make the following definition.
A vector $v$ is called an {\em $\F$-distance dominating} vector if $v\geq v(p)$ for some distribution $p$. 
Define $\Q_{\F}$ to be the set of all dominating distance vectors.
When $\F$ is the set of all measurable $\X\to[0,1]$ functions, we denote $\Q_{\F}$ by $\Q_{\tv}$.
\begin{claim}\label{c:convexup}
$\Q_{\F}$ is convex and upward-closed\footnote{Recall that upwards-closed means that whenever $v\in \Q_{\F}$ and $u\geq v$ then also $u\in \Q_{\F}$.}.
\end{claim}
\begin{proof}
That $\Q_{\F}$ is upward-closed is trivial.
Convexity follows since $d_\F$ is convex.
\end{proof}

The following claim shows that the non-trivial half-spaces that contain $\Q_{\F}$ have normals in the nonnegative orthant.
\begin{claim}\label{claim:separators}
If $h\in\R^n$ and $c\in \R$ satisfy that
$h\cdot v \geq c$ for all $v\in \Q_{\F}$,
then $h \geq 0$.
\end{claim}
\begin{proof}
We prove the contraposition.
Assume that $h_i < 0$ for some $i\leq n$.
then there is a vector $u$ with $u_i > u_j = 1$ for all $j$, where $u_i$ is sufficiently large so that $h\cdot u < c$.
The proof is finished by noting that such a $u$ satisfies $u\in\Q_{\F}$ (because it dominates any distance vector).
\end{proof}

\begin{corollary}\label{cor:separators}
Let $C$ be compact and convex such that $C\cap \Q_\F=\emptyset$.
Then, there is $h\geq 0$ such that
\[\max_{v\in C}h\cdot v < \min_{u\in\Q_{\F}}h\cdot u.\]
\end{corollary}
\begin{proof}
By the standard separation theorem for convex sets there is $h\in \R^n$ such that
$\max_{v\in C}h\cdot v < \min_{u\in\Q_{\F}}h\cdot u$. By Claim~\ref{claim:separators}
it follows that $h\geq 0$.
\end{proof}

Note that if $\F\subseteq \G$ are families of functions
then $\Q_{\G}\subseteq \Q_{\F}$. Thus, $\Q_{TV}\subseteq \Q_{\F}$ for every $\F$.
\begin{claim}
Let $\F$,$\G$ be families of $\X\to[0,1]$ functions.
The following two statements are equivalent:
\begin{enumerate}
\item $\Q_{\F}=\Q_{\G}$,
\item $\min_{v\in\Q_{\F}} h\cdot v = \min_{v\in\Q_{\G}} h\cdot v$, for every $h\geq 0$.
\end{enumerate}
\end{claim}
\begin{proof}
$1\implies 2$ is trivial.
For the other direction,
we prove the contraposition:
assume that $\Q_{\F}\neq \Q_{\G}$,
and without loss of generality that $u\in \Q_{\F}\setminus \Q_{\G}$.
Then, by Corollary~\ref{cor:separators}
there is $h\geq 0$ such that  $h\cdot u < h\cdot v$
for all $v\in \Q_{\G}$, and in particular,
$\min_{v\in\Q_{\F}} h\cdot v \neq \min_{v\in\Q_{\G}} h\cdot v$ as required.
\end{proof}

\section{Upper Bounds}\label{sec:ub}

In this section we show that every finite class $\Q$
	is $\alpha$-learnable for $\alpha=2$.
	This is achieved by \Cref{thm:ubdensities} and \Cref{thm:ubfinite} (stated below)
	which also provide quantitative bounds on the sample complexity.
	
\begin{theorem}[Upper bound infinite domain]\label{thm:ubdensities}
Let $\Q$ be a finite class of distributions over a domain $\X$ with $\lvert \Q\rvert = n$.
Then $\Q$ is $\alpha$-learnable with $\alpha=2$ and sample complexity
%\[m(\eps,\delta) =
%O\Bigl(\min\Bigl\{n\cdot \frac{\log(1/\delta)}{\eps^2}, \sqrt{n\log\log n}\cdot\frac{\log^{3/2}(1/\eps\delta)\log\log(1/\eps)}{\eps^{5/2}} \Bigr\}\Bigr).\]
\[m(\eps,\delta) =
\min\Bigl\{ O\Bigl( \frac{n+\log(1/\delta)}{\eps^2}\Bigr), \tilde O\Bigl(\sqrt{n}\cdot \frac{\log^{3/2}(1/\delta)}{\eps^{5/2}}\Bigr) \Bigr\}.\]

%\[\sqrt{n(\log n)^5(\log\log n)^4(\log\log\log n)^2}\cdot\frac{\log^{3/2}(1/\eps\delta)\log\log(1/\eps)}{\eps^{5/2}} \]
%\[O\Bigl(\min\bigl\{n, \sqrt{{n}\log\log(n/\eps)/{\eps}}\cdot\log^{3/2}(1/\eps\delta)\bigr\}\cdot \frac{\log(1/\delta)}{\eps^2}\Bigr).\]
\end{theorem}
The first bound of $O\bigl(\frac{n+\log(1/\delta)}{\eps^2}\bigr)$ gives a standard dependency on $\eps,\delta$
(standard in the sense that a similar dependence appear in popular concentration bounds).
The second bound improved the dependence on $n$ from linear to $\tilde O(\sqrt{n})$,
however it has inferior dependence with respect to $\eps.\delta$.
Both of these bounds depend polynomially on~$n$,
which is poor comparing to the logarithmic dependence exhibited by the proper $\alpha=3$ learning algorithm due to Yatracos.
The next theorem shows that for finite domains one can achieve a logarithmic dependence in $n$
(as well as in the size of the domain):

\begin{theorem}[Upper bound finite domain]\label{thm:ubfinite}
Let $\Q$ be a finite class of distributions over a finite domain $\X$ with $\lvert \Q\rvert = n$.
Then $\Q$ is $\alpha$-learnable with $\alpha=2$ and sample complexity
\[m(\eps,\delta) = O\Bigl(\frac{\log \frac{n}{\eps}\sqrt{\log\lvert \X\rvert}\log^{\frac{3}{2}}(1/\delta)}{\eps^3}\Bigr).\]
\end{theorem}

\Cref{thm:ubdensities} and \Cref{thm:ubfinite} are based on three algorithms,
	which are presented and analyzed in  \Cref{sec:adaptivealg} and \Cref{sec:staticalg} .
	In \Cref{sec:ubproofs} we use these algorithms to prove \Cref{thm:ubdensities} and \Cref{thm:ubfinite}.

%Throughout this section we let $Q=\{q_1,\ldots,q_n\}$
%and let $p^*$ denote the target distribution.
%As noted earlier, our algorithms are based on Lemma~\ref{lem:termination}.
%Specifically, we first find a distance vector $v\in \Q_{\tv}$ such that $v\leq v^* + \eps1_n$,
%where $v^* = v(p^*)$ is the distances-vector of the target distribution.
%Once we find such a $v^*$ we pick any $q$ such that $v(q) \leq v + \eps 1_n$.
%Lemma~\ref{lem:termination} implies that $\tv(q,p^*)\leq 2\opt +\eps$, as required.

\subsection{Adaptive Algorithms}\label{sec:adaptivealg}

In this section we present two algorithms which share a similar ``adaptive''  approach.
	These algorithms yield the sample complexity bounds with sublinear dependence on $n$:
	that is, the $\tilde O(\sqrt{n})$ bound from \Cref{thm:ubdensities}
	and the $\tilde O(\log n)$ bound from \Cref{thm:ubfinite}).
	The algorithm which achieves the $\tilde O(n)$ bound from \Cref{thm:ubdensities}
	is based on a ``static'' approach and appears in \Cref{sec:staticalg}.

The two adaptive algorithms can be extended to yield $\alpha=2$ learners for other metrics:
	they only rely on the triangle-inequality and some form of convexity (which allows to apply the Minimax Theorem).
	In particular they extend to any {\it Integral Probability Metric} (IPM)~\citep{muller1997integral}.

A crucial property that will be utilized in the sample complexity analysis is that these algorithms require only a {\it statistical query access} (which we define next) to the target distribution $p$;
	in a \underline{statistical query}, the algorithm submits a function $f:\X\to[0,1]$ to a {\it statistical query oracle} and receives back an estimate of $\Ex_{x\sim p}[f(x)]$.
	Note that the oracle can provide an $\eps$-accurate\footnote{That is, an estimate which is correct up to an additive error of $\eps$}
	estimate with a high probability by drawing~$O(1/\eps^2)$ samples from $p$ per-query and returning the empirical average of $f$ as an estimate.
%	Naively, this will require a sample complexity of $\tilde O(k/\eps^2)$ if one requires $\eps$-accurate estimates to $k$ {\it adaptive} queries.
	Interestingly, there are sophisticated methods within the domain of {\it Adaptive Data Analysis}
	that significantly reduce the amortized sample complexity
	for estimating $k$ adaptive queries~\citep{Dwork15adaptive,Bassily16stability}.
	We will use these results in our sample complexity analysis (in \Cref{sec:ubproofs}).
	
We prove the following:
\begin{theorem}\label{thm:sq}
Let $\Q=\{q_1,\ldots,q_n\}$ be a class of distributions, let $\eps > 0$, and let $p$ be the target distribution.
Then. there exist algorithms $A_1,A_2$ such that
\begin{enumerate}
\item $A_1$ makes at most $2n^2/\eps$ statistical queries to $p$ and satisfies the following:
if the estimates to all queries are $\eps/4$-accurate then it outputs $q$ such that $v(q)\leq v(p) +\eps$.
\item $A_2$ makes at most $2n\log n/\eps$ statistical queries to $p$ and satisfies the following:
if the estimates to all queries are $\eps/2\log n$-accurate then it outputs $q$ such that $v(q)\leq v(p) +\eps$.
\end{enumerate}
\end{theorem}

Note that by \Cref{lem:termination} it follows that the output distribution $q$ satisfies $\tv(q,p)\leq 2\opt + \eps$, as required.
%In the next section (\Cref{sec:ubproofs}) we will see how to combine these algorithm with tools from adaptive data anlaysis
%to prove our main results \Cref{thm:ubdensities} and \Cref{thm:ubfinite}.

%\todo{Perhaps note that this algorithm can be modified
%to satisfy differential privacy.
%This can be seen as a certificate for {\it non memorization}}
%
%\todo{Observe that this algorithm doesn't use the specifics of total variation and could be extended to other probability metrics that have the form of an IPM}

\begin{figure}
\begin{tcolorbox}
\begin{center}
{\bf A statistical query approach for $\alpha=2$ learning finite distributions}\\
\end{center}
\noindent
Given: A class $\Q=\{q_1,\ldots,q_n\}$, and a sampling access to a target distribution $p$ and $\eps,\delta>0$.\\
Output: A distribution $p_0$ such that $\tv(p_0,p) \leq 2\min_{i}\tv(q_i,p) + \eps$
with probability at least $1-\delta$.
\begin{enumerate}
\item Let $v^*=v(p)=(\tv(p,q_i))_i\in\R^n$, and set $y^0=(0,\ldots,0)\in\R^n$. (Note that $v^*$ is not known)
%\item Sample $m$ points from $p$. (Upper bounds on $m$ are presented in Section~\ref{sec:sampcomplexity}.)
\item For $k=1,\ldots$
\begin{enumerate}
	\item If $y^{k} + \eps\cdot 1_n\in\Q_{\tv}$ then output $p'$
	such that $\tv(p',q_i)\leq y^{k}_i + \eps$ for $i=1,\ldots,n$.
	\item Else,  find an index $j$ such that
	$y_k + \frac{\eps}{2}e_j \leq v^*$,
	set $y^{k+1} = y^k + \frac{\eps}{2}e_j$, and continue to the next iteration.
\end{enumerate}
\end{enumerate}
\end{tcolorbox}
\caption{Both algorithms $A_1,A_2$ follow this pseudo-code.
They differ in item 2(b) which is implemented differently in each of them;
$A_1$ uses more statistical queries than $A_2$ but $A_2$ requires less accuracy-per-query than~$A_1$.}
\label{alg:online}
\end{figure}

\begin{proof}[Proof of \cref{thm:sq}]

Both algorithms $A_1,A_2$ follow the same skeleton which is depicted in Figure~\ref{alg:online}.
	The approach is based on Lemma~\ref{lem:termination}
	by which it suffices to find a vector~$y\in\Q_{\tv}$ such that $y \leq v^* + \eps\cdot1_n$, where $v^*=v(p)$
	is the distance vectors of the target distribution $p$ with respect to the~$q_i$'s.
	The derivation of such a distance-vector $y$ is based on the convexity of $\Q_\tv$,
	and the access of the algorithms to $\Q_\tv$ can be conveniently abstracted via the following separation oracle:
\begin{definition}[Separation oracle]
A separation oracle for $\Q_{\tv}$ is an algorithm which, given  an input point $v\in\R^n$,
if $v\in \Q_{\tv}$ then it returns $q$ such that $v(q)\leq v$, and otherwise, it returns a hyperplane separating $v$ from $\Q_{\tv}$.
\end{definition}
The separation oracle is used in item 2.

The derivation of the desired distances-vector $y$ is achieved by producing an increasing sequence of vectors
	\[0 = y^0\leq y^1\leq y^2\leq\ldots\leq v^*,\]
	such that $y^{k+1}$ is obtained from $y^k$ by increasing a carefully picked coordinate $j$ by $\eps/2$ (in item 2(b)).
	We postpone the details of how $j$ is found and first assume it in order to argue that total number of iterations is at most $O(n/\eps)$:
	indeed, observe that the $\|y^k\|_1$ increases by $\eps/2$ in each step (i.e.\ $\|y^k- y^{k-1}\|\geq \eps/2$).
	Thereofore,  since $\|y^k\|_1\leq \|v^*\|_1\leq n$ we see that after at most $t\leq 2n/\eps$ steps,
	$y^t$ must satisfy $y^t+ \eps\cdot 1\in \Q_{\tv}$.
	In this point a distribution $q$ is outputted such that $v(q) \leq y^t + \eps\cdot 1\leq  v(p) + \eps\cdot 1_n$,
	as required.

It thus remains to explain how an appropriate index $j$ is found in item 2(b) (which is also where the implementations of $A_1,A_2$ differs).
	The derivation of $j$ follows via an application of LP duality (in the form of the Minimax Theorem) as we explain next.

\subsubsection{Finding an index in each step}\label{sec:binary}

Consider an arbitrary step in the algorithm, say the $k$'th step.
Thus, we maintain a vector $y^k$ that satisfies $y^k\leq v^*$.
We assume that $y^k + \eps\cdot 1_n \notin \Q_{\tv}$ (or else we are done),
and we want to show how, using few statistical queries, one can find an index $j$
such that $y^k + \frac{\eps}{2}e_j\leq v^*$.

The following lemma is the crux of the argument.
On a high level, it shows how using a few statistical queries,
one can estimate a vector $\hat z=\hat z(p)\in\R^n$ such that (i) $\hat z\leq v^*$,
and (ii) there is an index $j$ such that $y^k_j + \frac{\eps}{2}\leq \hat z_j$.
This means that the index $j$ satisfies the requirements,
and we can proceed to the next step by setting $y^{k+1} = y^k + \frac{\eps}{2}e_j$.
\begin{lemma}\label{lem:progress}
Let $y\in\R^n$ such that $y\notin \Q_{\tv}$.
Then, there are $n$ functions $F_i:\X\to[0,1]$,
and $n$ coefficients $h_i\geq 0$ with $\sum_i h_i=1$,
such that for every distribution $p$ the vector $z=z(p)$,
defined by $z_i = \Ex_p[F_i] - \Ex_{q_i}[F_i]$, satisfies:
\begin{enumerate}
\item $\sum_{i}h_i\bigl(z_i - y_i\bigr) > 0$, and
\item $z_i \leq \tv(p,q_i)$ for all $i$.
\end{enumerate}
\end{lemma}
We stress that the $n$ functions $F_i$'s depend only on the $q_i$'s and on $y$.
\begin{proof}[Proof of Lemma~\ref{lem:progress}]
First, use Corollary~\ref{cor:separators} to find $h\geq 0$,
such that $\sum_i{h_i y_i} < \min_{v\in\Q_{\tv}}\sum_i{h_iv_i}$.
Note that necessarily $h\neq 0$, and therefore we can normalize it so that $\sum_{i}h_i=1$.
Next, we find the functions $F_i$'s using the Minimax Theorem~\citep{Neumann1928}:
\begin{align*}
\sum_i{h_iy_i} < \min_{u\in \Q_{\tv}} \sum_i h_iu_i
   &= \min_{p\in\Delta(\X)}\max_{f_i:\X\to[0,1]}\sum_ih_i(\Ex_p[f_i] - \Ex_{q_i}[f_i])\\
   &= \max_{f_i:\X\to[0,1]}\min_{p\in\Delta(\X)}\sum_ih_i(\Ex_p[f_i] - \Ex_{q_i}[f_i])\,.
\end{align*}
Pick the functions $F_i$'s to be maximizers of the last expression (i.e.\ the maximizers of $\min_{p\in\Delta(\X)}(\Ex_p[f_i] - \Ex_{q_i}[f_i])$).
Therefore, $\sum_i{h_iy_i}\leq \sum_ih_i(\Ex_p[f_i] - \Ex_{q_i}[f_i])$ for every distribution $p$.
This is equivalent to $\sum_{i}h_i\bigl(z_i - y_i\bigr) > 0$, which is the first item of the conclusion.
For the second item, note that
\[z_i = \Ex_p[F_i] - \Ex_{q_i}[F_i]\leq \max_{f_i:\X\to[0,1]}\Ex_p[f_i] - \Ex_{q_i}[f_i] = \tv(p,q_i), \]
as required.
\end{proof}

\begin{figure}
\begin{tcolorbox}
\begin{center}
{\bf Binary search}\\
\end{center}
\noindent
Input: vectors $y,h$, and $n$ functions $F_i$ as in Lemma~\ref{lem:progress}, and a sample access to the target distribution $p$.\\
Output: an index $j$ such that $y+\frac{\eps}{2}e_j \leq v^*$.
\begin{enumerate}
\item Set $n_{min}=1, n_{max}=n$.\\
While $n_{min} < n_{max}$:
\begin{enumerate}
	\item Set $n_{mid} = \lfloor\frac{n_{min} + n_{max}}{2}\rfloor$, $\ell= \sum_{i=n_{min}}^{n_{mid}}h_i$, $u=\sum_{i=n_{mid}+1}^{n_{max}}h_i,$ and
	\[L(x) = (1/\ell)\sum_{i=n_{min}}^{n_{mid}}h_i\bigl(F_i(x) - \Ex_{q_i}[F_i]-y_i\bigr)\]  \[U(x) = (1/u)\sum_{i=n_{mid+1}}^{n_{max}}h_i\bigl(F_i(x)- \Ex_{q_i}[F_i]-  y_i\bigr).
	\]
	\item Submit statistical queries to derive estimates $ \hat \mu_L, \hat \mu_U$ of $\Ex_p[L(x)],\Ex_p[U(x)]$ respectively up to an additive error of~$\frac{\eps}{2\log n}$.
	\item If $\hat \mu_L \geq \hat \mu_U$ then set $n_{min}= n_{min}$, $n_{max} = n_{mid}$, and normalize $h_i= \frac{h_i}{\ell}$ for $n_{min}\leq i \leq n_{max}$
	and else set $n_{min} = n_{mid}+1$, $n_{max}=n_{max}$, and normalize $h_i= \frac{h_i}{u}$ for $n_{min}\leq i \leq n_{max}$.
\end{enumerate}
\item Output $n_{min}$ ($=n_{max}$).
\end{enumerate}
\end{tcolorbox}
\caption{Binary search for an appropriate index $i$}
\label{alg:binary}
\end{figure}

We next show how to use Lemma~\ref{lem:progress} to find an appropriate index $j$.
Plug in the lemma~$y=y^k + \eps\cdot 1_n$, and set $z=z(p)$, where $p$ is the target distribution.
Note that since the $F_i$'s are known, we can use statistical queries for~$\Ex_{p}[F_i]$'s to estimate the entries of $z$.
By the first item of the lemma:
\[ \sum_{i}h_i\bigl(z_i - y^k_i - \eps\bigr) \geq 0 \implies  \sum_{i}h_i\bigl(z_i - y^k_i \bigr) \geq \eps,\]
which implies that there exists an index $j$ such that $y^k_j + \eps \leq z_j$
(in fact it shows that if we interpret the $h_i$'s as a distribution over indices $i$ then,
on average, a random index will satisfy it).
The second item implies that increasing such a coordinate $j$ by $\eps$ will keep it upper bounded $v^*_j$.

Thus, it suffices to estimate each coordinate $z_i$ up to an additive error of $\eps/4$,
and pick any index $j$ such that the estimated value satisfies $\hat z_j \geq 3\eps/4$.
$A_1$ achieves this simply by querying $n$ statistical queries (one per $F_i$) with accuracy $\eps/4$.
So, the total number of statistical queries used by $A_1$ is at most $\frac{n}{\eps}\cdot n$,
and if each of them is $\eps/4$-accurate then it outputs a valid distribution $q$.

It remains to show how $A_2$ finds an index $j$.
$A_2$ uses a slightly more complicated binary-search approach,
which uses just $\log n$ statistical queries, but requires higher accuracy of $\eps/4\log n$.

\paragraph{Binary search for an appropriate index $i$.}
%It will turn out fruitful to consider a more convoluted way of finding an index $j$.
The pseudo-code appears in Figure~\ref{alg:binary}.
We next argue that the index $j$ outputted by this procedure satisfies $z_j-y_j\geq \eps/2$.
Consider the first iteration in the while loop; note that
$\Ex_p[L(x)] = (1/\ell)\sum_{i} h_i(z_i - y_i),~~\Ex_p[U(x)] = (1/u)\sum_{i} h_i(z_i - y_i)$.
Therefore, since $\eps \leq \sum_i{h_i(z_i-y_i)}$ it follows that
$\eps \leq \sum_i{h_i(z_i-y_i)} = \ell\Ex_p[L(x)] + u\Ex_p[U(x)]$.
Now, $\ell+u=1$, and therefore $\max\{\Ex_p[L(x)],\Ex_p[U(x)]\}$ is at least $\eps$.
This in turn implies that $\max\{\hat \mu_L, \hat \mu_U\}$ is at least $\eps-\frac{\eps}{2\log n}$.
Therefore, in the second iteration we have $\sum_{i=n_{min}}^{n_{max}}h_i(z_i-y_i)\geq \eps-\frac{\eps}{2\log n}$.
By applying the same argument inductively we get that at the $m$'th iteration we have
$\sum_{i=n_{min}}^{n_{max}}h_i(z_i-y_i)\geq \eps-\frac{m\cdot\eps}{2\log n}$,
and in particular in the last iteration we find an index $j$ such that $z_j-y_j\geq \eps/2$,
as required.

\end{proof}

\section{A Static Algorithm}\label{sec:staticalg}

\paragraph{Uniform convergence.}
Before we describe the main result in this section
we recall some basic facts from statistical learning theory that will be useful.
Let $\F$ be a class of functions from $\X\to[0,1]$.
We say that $\F$ has uniform convergence rate of (at most) $d$
if for every distribution $p$ over $\X$ and every $m\in\N,\delta\in (0,1)$,
\[\Pr_{S\sim p^m}\Bigl[\sup_{f\in \F}\lvert p(f) - p_S(f)\rvert > \sqrt{\frac{d+\log(1/\delta)}{m}}\Bigr]\leq \delta.\]

It is well known that if $\F$ is a class of $\X\to\{0,1\}$ functions with VC dimension $d$
then its uniform convergence rate is $\Theta(d)$~\cite{Vapnik71uniform}.

\begin{lemma}\label{lem:union}
Let $\F_1,\ldots, \F_d$ be classes with VC dimension at most $d$.
Then, the VC dimension of $\cup_i \F_i$ is at most $10d$.
\end{lemma}
\begin{proof}
We show that $\cup_i \F_i$ does not shatter a set of size $10d$.
Let $Y\subseteq \X$ of size $100d$.
Indeed, by the Sauer-Shelah Lemma~\cite{sauer72sauer}:
\[\bigl\lvert (\cup_i \F_i) |_Y \bigr\rvert  \leq d{100d \choose \leq d} \leq d 2^{10dh(1/10)} < 2^{10d},\]
where $(\cup_i \F_i) |_Y = \{f \cap Y : f\in \cup_i \F_i \}$,
and the second to last inequality follows by a standard upper bound on the binomial coefficients
by the entropy function: ${n \choose k}\leq 2^{nh(k/n)}$ for every $k\leq n$, where $h(p) =  - p\log p  -  (1-p)\log(1-p)$.
\end{proof}

We next present the main result in this section which is an algorithm which achieves factor $2$
whose sample complexity is $O(\frac{n+\log(1/\delta)}{\eps^2})$.
It is conceptually simpler than the adaptive algorithms from the previous section (although the proof here is more technical).
Specifically, it is based on finding a set $\F$ of $\X\to[0,1]$ functions
which satisfies two properties:
\begin{itemize}
\item[(i)] Given some $O(\frac{n+\log(1/\delta)}{\eps^2})$ samples from $p$,
one can estimate $d_\F(p,\cdot)$ up to an additive $\eps$ error, with probability at least $1-\delta$ (where the probability is over the samples from $p$). In particular this means that the distance vector $v^*_{\F}=v_{\F}(p)$
of $p$ with respect to $\F$ can be estimated from this many samples.
\item[(ii)] $\tv$ and $d_\F$ have the same distances vectors, i.e.\ $\Q_{\F} = \Q_{\tv}$.
\end{itemize}
Using these two items the algorithm proceeds as follows:
it uses the first item to estimate $v^*_\F=v_\F(p)$ up to an additive~$\eps$.
Then, it uses the second item (by which $v^*_\F\in\Q_\tv$) to find $q$ such that $v(q) \leq v^*_\F + \eps \leq v^* +\eps$
and outputs it. \Cref{lem:termination} then implies that $\tv(q,p)\leq 2\opt + \eps$ as required.

\begin{theorem}\label{thm:static}
Let $\Q=\{q_1,\ldots,q_n\}\subseteq\Delta(\X)$.
Then there exists a class $\F=\F(Q)$ of functions from $\X$ to $\{0,1\}$
such that:
\begin{enumerate}
\item $\Q_{\tv} = \Q_{\F}$,  and
\item The VC dimension of $\F$ is at most $10n$ (in particular, the uniform convergence rate of $\F$ is some $O(n)$).
\end{enumerate}
\end{theorem}

\paragraph{Construction of $\F$.}
Consider the Yatracos functions $S_{i,j}:\X\to\{0,1\}$ that are defined by $S_{i,j}(x) = 1$ if and only if~$q_i(x)\geq q_j(x)$,
and define \[\F_i = \{ {\bf 1}_{\sum_{j\neq i} h_j S_{i,j} \geq c} : h_j, c\in \R\}.\]
The class $\F$ is defined by
\[ \F = \cup_i \F_i.\]
See \Cref{fig:majorityschaffe} for an illustration of a function in $\F$.

\begin{figure}
\begin{center}
\includegraphics[width=.75\textwidth]{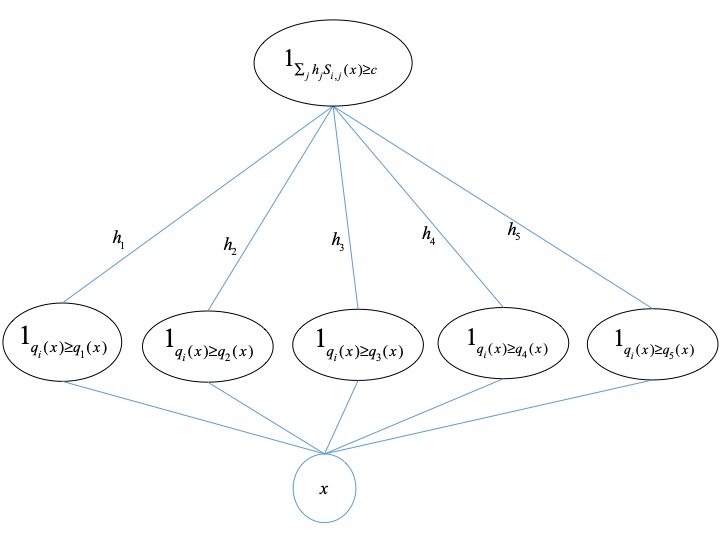}
\end{center}
\caption{An illustration of a function in $\F$.}
\label{fig:majorityschaffe}
\end{figure}

Theorem~\ref{thm:static} follows from the next two lemmas (\Cref{lem:polyt} implies that $\Q_\F=\Q_\tv$ via \Cref{cor:separators}).
\begin{lemma}\label{lem:vc}
$\F$ has VC dimension at most $10n$.
\end{lemma}

\begin{lemma}\label{lem:polyt}
For every $h\geq 0$
\[\min_{v\in\Q_{\tv}} h\cdot v = \min_{v\in\Q_{\F}} h\cdot v\]
\end{lemma}

\begin{proof}[Proof of Lemma~\ref{lem:vc}]
We claim that the VC dimension of each $\F_i$ is at most $n$,
this will finish the proof by Lemma~\ref{lem:union}.
To see that $\F_i$ has VC dimension at most $n$,
we show that its {\it sign-rank} (defined below) is at most $n$.
This implies the bound on the VC dimension,
since the VC dimension is at most the sign-rank (see e.g.~\citep{Alon16signrank}).

The sign-rank of $\F_i$ is the minimal $d$
such that there is a representation of $\X$ using $d$-dimensional vectors
so that each $f\in\F_i$ corresponds to a $d$-dimensional half-space.
Formally, if there is a mapping $\phi:\X\to\R^d$
such that for every $f\in \F_i$ there is $u\in \R^d$
such that $f(x)=1$ if and only if $u\cdot \phi(x) \geq 0$.

To see that the sign-rank of $\F_i$ is at most $n$
consider the mapping
\[\phi(x) = \bigl(S_{i,1}(x),\ldots S_{i,i-1}(x),S_{i,i+1}(x),\ldots S_{i,n}(x) , 1\bigr)\in\R^n.\]
For every $f\in \F$ with $f={\bf 1}_{\sum_{j\neq i} h_j S_{i,j} \geq c}$
pick $v\in \R^n$ where the first $n-1$ coordinates of $v$
are the $h_j$'s for $j\neq i$, and the last coordinate is $-c$.
The half-space defined by $u$ indeed corresponds to $f$:
\[ f(x) = 1 \iff {\bf 1}_{\sum_{j\neq i} h_j S_{i,j} \geq c}(x) = 1 \iff   \sum_{j\neq i} h_j S_{i,j} \geq c \iff v\cdot \phi(x) \geq 0.\]

\end{proof}

\begin{proof}[Proof of Lemma~\ref{lem:polyt}]

\Cref{lem:polyt} follows by a careful inspection of the vertices of $\Q_\tv$.
	This inspection involves a somewhat technical analysis of the solutions of a related linear program.
	We provide the proof in \Cref{sec:lempolyt}.
\end{proof}

\subsection{Proofs of \Cref{thm:ubdensities} and \Cref{thm:ubfinite}}\label{sec:ubproofs}

\Cref{thm:ubdensities} and \Cref{thm:ubfinite} follow from \Cref{thm:sq} and \Cref{thm:static},
	combined	with results in Adaptive Data Analysis.
	We refer the reader to the survey by~\cite{Dwork15adaptive} for a detailed introduction.

First, the $O(\frac{n+\log(1/\delta)}{\eps^2})$ bound in \Cref{thm:ubdensities} is a direct corollary
	of the static algorithm from the previous section (see the discussion prior to \Cref{thm:static}'s statement).
	The second bound in \Cref{thm:ubdensities} and the bound in \Cref{thm:ubfinite} follows from the two adaptive algorithms $A_1,A_2$ in \Cref{thm:sq},
	as we explain next.

In order for Algorithms $A_1,A_2$ to output a valid distribution $q$,
	it is required that all of the statistical queries they use are answered with the desired accuracy.
	Recall that $A_1$ uses $2n^2/\eps$ queries and requires accuracy of $\eps/4$ per query
	and that $A_2$ uses $2n\log n/\eps$ queries and require accuracy of $\eps/2\log n$ per query.
	To achieve this, one needs to draw enough samples from the target distribution $p$
	that suffice for a good-enough estimate.
	A natural way is to estimate each of the statistical queries by its empirical average.
	However, since the algorithm is adaptive
	(i.e.\ the choice of the statistical query used in iteration $k$ depends on the previous queries and their estimates),
	this may require a large number of samples from $p$.
	In particular, there are settings in which if one uses the empirical averages as estimates
	then $\Omega(k/\eps^2)$ samples are needed in order to answer $k$ adaptive queries adaptively
	Luckily, the domain of {\it Adaptive Data Analysis} has developed clever estimates
	which achieve significant reductions in the sample complexity.
	In a nutshell, the idea is to return a noisy version of the empirical averages,
	and the high-level intuition is that the noise stabilizes this random process and hence makes it more concentrated.

We will use the following results due to~\cite{Bassily16stability}, which improve upon results from~\cite{Dwork15adaptive}.
%	
%	
%	By now, there are various upper bounds on the number of samples
%	that are enough to perform $k$ (adaptive) estimates of
%	expectations of ``$X\to [0,1]$'' functions, see~\cite{} and references within for more details.
%	In particular, they yield the following upper bounds on the sample complexity in our context:
\begin{theorem}[Infinite domain, Corollay 6.1 in~\cite{Bassily16stability}]\label{thm:adainf}
Let $p$ be the target distribution.
	Then, there is a mechanism that given $n=n(\eps,\delta)$ samples from $p$,
	 answers $k$ adaptive statistical queries such that with probability at least $1-\delta$
	 each of the provided estimates is $\eps$-accurate, and
	\[n(\eps,\delta) = O\Bigl(\frac{\sqrt{k\log\log k}\log^{3/2}(1/\eps\delta)}{\eps^2}\Bigr).\]
%
%
%There are estimation mechanisms that, with probability of at least $1-\delta$,
%provide accuracy of at least~$\frac{\eps}{2\log n}$, simultaneously
%to all expectations estimated by the algorithm by using at most
%\[\min\Bigl\{\tilde O\Bigl(\frac{\sqrt{n}}{\eps^{5/2}}\log(1/\delta)\Bigr), \tilde O\Bigl(\frac{\sqrt{\log\lvert X\rvert}\log n}{\eps^{7/2}}\Bigr)\Bigr\}\]
%samples from the target distribution $p^*$.
\end{theorem}

\begin{theorem}[Finite domain Corollary 6.3 in~\cite{Bassily16stability}]\label{thm:adafin}
Let $p$ be the target distribution.
	Then, there is a mechanism that given $n=n(\eps,\delta)$ samples from $p$,
	 answers $k$ adaptive statistical queries such that with probability at least $1-\delta$
	 each of the provided estimates is $\eps$-accurate, and
	\[n(\eps,\delta) = O\Bigl(\frac{\sqrt{\log \lvert \X\rvert}\log k\log^{3/2}(1/\eps\delta)}{\eps^3}\Bigr).\]

Algorithm $A_2$ combined with \Cref{thm:adainf} yields the $\tilde O(\sqrt{n})$ dependence in \Cref{thm:ubdensities},
and $A_1$ combined with \Cref{thm:adafin} yields \Cref{thm:ubfinite}.
%
%
%There are estimation mechanisms that, with probability of at least $1-\delta$,
%provide accuracy of at least~$\frac{\eps}{2\log n}$, simultaneously
%to all expectations estimated by the algorithm by using at most
%\[\min\Bigl\{\tilde O\Bigl(\frac{\sqrt{n}}{\eps^{5/2}}\log(1/\delta)\Bigr), \tilde O\Bigl(\frac{\sqrt{\log\lvert X\rvert}\log n}{\eps^{7/2}}\Bigr)\Bigr\}\]
%samples from the target distribution $p^*$.
\end{theorem}

\section{Lower Bounds}\label{sec:lb}
As discussed in the introduction, any finite $\Q$ can be properly $\alpha=3$-learned by Yatracos' algorithm.
We show that $\alpha=3$ is optimal:
\begin{theorem}[Lower bound for infinite domains]\label{thm:lbdensities}
For arbitrarily small $0<\beta < 1$ there is a class $\Q=\Q(\beta) = \{q_1,q_2\}$ of two densities such that the following holds.
Let $A$ be a {(possibly randomized)} proper learning algorithm for $\Q$ and let $m$ be a sample complexity bound.
Then, there exists a target distribution $p$ such that {$\opt=\frac{1}{2}\beta$} and if $A$ gets at most $m$ samples from $p$ as an input then
\[
TV(q,p) > 3\cdot\frac{1}{2}\beta -2\beta^2 = (3-6\beta)\opt + \beta^2\,,
\]
with probability at least $\frac{1}{3}$.
%In particular, put $\alpha < 3$, then no proper algorithm can achieve a guarantee of $\alpha\cdot\opt + \eps$ for $\eps< 1-\alpha/3$.
\end{theorem}

The following corollary summarizes that $\alpha=3$ is the threshold for proper learning.
\begin{corollary}
For every $\alpha<3$ there exists $\eps_0>0$ and a class $\Q$ containing two densities such that 
no proper algorithm can agnostically learn $\Q$
 with a guarantee of at most
\[\alpha\cdot\opt + \eps_0,\]
and success probability $\delta>2/3$.
\end{corollary}
\begin{proof}
Let $\alpha<3$. The proof follows from \Cref{thm:lbdensities} by picking  $0< \beta <  \frac{3-\alpha}{6}$ , setting $\eps_0=\beta^2$, and noting that $(3-6\beta)\opt + \beta^2 = \alpha\cdot\opt + \eps_0$.
\end{proof}

%For example, say $\alpha=2.9$, then no proper algorithm can achieve an $\alpha\cdot\opt + \eps$ guarantee where $\eps< 0.01$.

%given any $\alpha<3$, by picking $\epsilon=1-\alpha/3$ the above theorem yields a lower bound of  $3\opt-2\opt^2\ge 3\opt-3\opt^2 = \alpha\cdot\opt$.

For finite domains we get the next version of \Cref{thm:lbdensities} which gives a quantitative sample complexity lower bound.

%\begin{theorem}[Lower bound for finite domains]\label{thm:lbfinite}
%Let $\X$ be a domain of size $N$.
%Then, for every $\beta < 1$ there is a class $\Q=\Q(\beta) = \{q_1,q_2\}$ of two densities such that the following holds.
%Let $A$ be a {(possibly randomized)} proper learning algorithm for $\Q$.
%Then, there exists a target distribution $p$ such that $\opt=\beta$ and if $A$ gets at most $\sqrt{N}$ samples from $p$ as an input then
%\[
%TV(q,p) \ge 3\beta -2\beta^2 = (3-3\beta)\opt + \beta^2\,,
%\]
%with probability at least  $\frac{1}{3}$.
%\end{theorem}

\begin{theorem}[Lower bound for finite domains]\label{thm:lbfinite}
{Let $\beta\in (0,1)$ such that $\frac{1+\beta}{\beta}\in\mathbb{N}$ (note that there are arbitrarily small such $\beta$'s), and let $\X$ be a domain of size $M > 1/\beta$.}
Then, there exist two densities 
$\Q=\{q_1,q_2\}$ over $\X$ such that the following holds.
For any (possibly randomized) proper learning algorithm $A$ for~$\Q$,
there exists a target distribution~$p$ with {$\opt=\frac{1}{2}\beta$} such that,
if $A$ receives at most {$\sqrt{M\beta}$} samples from $p$, then
with probability at least~$\tfrac13$ the returned hypothesis $q$ satisfies
\[
  \tv(q,p)
  \;>\;
3\cdot  \frac{1}{2}\beta - 2\beta^2
  \;=\;
  (3-6\beta)\opt + \beta^2 .
\]
\end{theorem}

We will make use of the following lemma which is a simple generalization of Le Cam's Lemma (see \cite{yu1997assouad}, Lemma~1)
\begin{lemma}\label{le:lecam}
Let $\D_1$ and $\D_2$ be two families of probability distributions, $\D_i^{\oplus m}$ denotes the distribution obtained by sampling $p\sim\D_i$ (assuming some given fixed distribution over $\D_i$) and then drawing~$m$ independent samples from $p$. Consider an algorithm (which can be randomized) that determines, given $m$ i.i.d.\ examples from some $p\in\D_1\cup \D_2$, whether $p\in\D_1$ or $p\in\D_2$. Then such an algorithm will  have a probability of making a mistake  lower bounded by
\[
\frac{1}{2}\left(1-\tv(\D_1^{\oplus m},\D_2^{\oplus m})\right)
\]
\end{lemma}
\begin{proof}
We first assume that the algorithm is deterministic.
Any deterministic algorithm deciding whether $p$ comes from $\D_1$ or $\D_2$ is associated with a set $A\subseteq \X^m$ (the set such that if the sample falls in it, it decides $i=1$, and $i=2$ otherwise). The worst-case probability of the algorithm to err is given by
\[
\max\paren*{\max_{p\in\D_2} p^m(A), \max_{p\in\D_1} p^m(\overline{A})}
\]
which can be lower bounded by the expectation under first choosing between $i=1$ and $i=2$ with probability $1/2$ and then picking $p\sim \D_i$:
\[
\frac{1}{2}\left(E_{p\sim\D_1} p^m(A) + E_{p\sim\D_2} p^m(\overline{A}) \right) = \frac{1}{2}\left(1+\D_1^{\oplus m}(A) - \D_2^{\oplus m}(A)\right) \ge \frac{1}{2}\left(1-\tv(\D_1^{\oplus m},\D_2^{\oplus m})\right)\,.
\]
If the algorithm is randomized then it may pick $A$ randomly,
so there is an additional expectation with respect to the distribution over sets $A$ which also leads to the same lower bound.
\end{proof}

The following lemma is of independent interest and can be seen as a chain rule for total variation. It essentially says that two distributions are close if there exists an event $E$ with large probability under each of those distributions and such that, conditioned on this event, the two probability distributions are close.
\begin{lemma}\label{le:chain}
Given two probability distributions $P,Q$ on a domain $\X$ and an event $E\subset \X$, denoting by $P_{|E}$ and $\Q_{|E}$ the corresponding conditional distributions (i.e. $P_{|E}(A):= P(A|E)$), we have
\[
\tv(P,Q) \le \tv(P_{|E},Q_{|E}) + 2P(\overline{E})+2Q(\overline{E})
\]
\end{lemma}
\begin{proof}
\begin{eqnarray*}
\tv(P,Q) = \sup_{A} \abs*{P(A)-Q(A)} &\le&
\sup_{A} \abs*{P(A\cap E)-Q(A\cap E)} + \sup_{A} \abs*{P(A\cap \overline{E})-Q(A\cap \overline{E})}\\
&\le&
 \sup_{A} \abs*{P(E)\paren*{P(A| E)-Q(A| E)} + Q(A|E)\paren*{P(E)-Q(E)}}\\
 &~& + P(\overline{E})+Q(\overline{E})\\
&\le&
 P(E)\sup_{A} \abs*{P(A| E)-Q(A| E)} + \abs{P(E)-Q(E)}\\
  &~&+ P(\overline{E})+Q(\overline{E})\\
 &\le& \tv(P_{|E},Q_{|E}) + \abs{P(E)-Q(E)} + P(\overline{E})+Q(\overline{E})\\
&=& \tv(P_{|E},Q_{|E}) + \abs{P(\overline{E})-Q(\overline{E})} + P(\overline{E})+Q(\overline{E})\\
&\le& \tv(P_{|E},Q_{|E}) + 2P(\overline{E})+2Q(\overline{E})\\
\end{eqnarray*}
\end{proof}

\begin{proof}[Proof of \Cref{thm:lbdensities} and \Cref{thm:lbfinite}]
We first prove \Cref{thm:lbdensities} and later note how the proof can be modified to obtain \Cref{thm:lbfinite}.

Fix $\beta \in (0,1)$ and an integer $m \ge 1$.  
We will define two fixed densities $q_1, q_2$ over $[0,1]$, and two finite families~$\mathcal{D}_1, \mathcal{D}_2$ of probability densities over $[0,1]$, such that the following
three properties hold.

%\begin{enumerate}
%\item[(i)] For every $p \in \mathcal{D}_1$ we have
%\[
%\tv(q_1,p)=\beta
%\qquad\text{and}\qquad
%\tv(q_2,p) \;\ge\; 3\beta - 2\beta^2 .
%\]
%
%\item[(ii)] For every $p \in \mathcal{D}_2$ we have
%\[
%\tv(q_2,p)=\beta
%\qquad\text{and}\qquad
%\tv(q_1,p) \;\ge\; 3\beta - 2\beta^2 .
%\]
%
%\item[(iii)] The two meta-distributions obtained by drawing $p$ uniformly from 
%$\mathcal{D}_1$ (respectively $\mathcal{D}_2$) and then drawing $m$ i.i.d.\ samples from $p$
%satisfy
%\[
%\tv\!\left( \mathcal{D}_1^{\oplus m},\, \mathcal{D}_2^{\oplus m} \right) 
%\;\le\; \frac{1}{3} .
%\]
%\end{enumerate}
%
%Let $\beta<1$ and $m\in\mathbb{N}$; the proof follows by constructing $\Q=\{q_1,q_2\}$ and
%	two families of distributions~$\D_1,\D_2$ with the following properties:
\begin{itemize}
%	\item If $p\in \D_1$ then $\tv(q_2,p) > 3(1-\beta)\cdot\tv(q_1,p) + \beta^2$.
	\item If $p\in \D_1$ then $\tv(q_1,p)=\frac{1}{2}\beta$ and $\tv(q_2,p) > 3\cdot\frac{1}{2}\beta - 2\beta^2$.
%	\item If $p\in \D_2$ then $\tv(q_1,p) > 3(1-\beta)\cdot\tv(q_2,p) + \beta^2$.
	\item If $p\in \D_2$ then $\tv(q_2,p)=\frac{1}{2}\beta$ and $\tv(q_1,p) > 3\cdot\frac{1}{2}\beta - 2\beta^2$.
	\item $\tv(\D_1^{\oplus m},\D_2^{\oplus m}) \le 1/3$,
	where $\D_i^{\oplus m}$ denotes the distribution obtained by sampling $p$ uniformly from $\D_i$
	and then taking $m$ independent samples from $p$.
\end{itemize}
%Given that these properties hold, \Cref{thm:lbdensities} follows via \Cref{le:lecam}:
%	Indeed, we
To see how these 3 items conclude the proof of \Cref{thm:lbdensities}, consider the following game between an {\it adversary} and a {\it distinguisher}:
	the adversary randomly picks one of  $\D_1^{\oplus m},\D_2^{\oplus m}$, each with probability $1/2$,
	and draws a random sample $\vec x$ from it. Then, it shows $\vec x$ to the distinguisher, whose goal is to determine
	whether $\vec x$ was drawn from  $\D_1^{\oplus m}$ or from $\D_2^{\oplus m}$.

Now, by the first two properties, it follows that any (possibly randomized)
	proper learning algorithm for $\Q$ that uses an input sample of size $m$
	and outputs $q_i$ such that $\tv(q_i,p)\leq 3\opt -2\opt^2 $ with confidence $1-\delta$
	can be used by the distinguisher to guarantee
	a failing probability of at most $\delta$.
	However, since by Lemma \ref{le:lecam} any distinguisher fails with probability at least $1/2 - \tv(\D_1^{\oplus m},\D_2^{\oplus m})/2$,
	the third property implies that~$\delta \geq1/3$ as required.

%The third property implies that one cannot distinguish whether a random sample is drawn from $\D_1^{\oplus n}$ or from $\D_2^{\oplus n}$
%	with confidence $\geq 2/3 $ which (by the first two properties)
%	is an easier task than properly $\alpha$-learning $Q$ with $\alpha=3(1-\beta)$ and confidence $\delta\geq 2/3$.

\paragraph{Construction.}
We first define two reference densities $q_1,q_2$ over $[0,1]$ by
\[
q_1(x) =
\begin{cases}
1-\beta & x < 1/2,\\[2mm]
1+\beta & x \ge 1/2,
\end{cases}
\qquad\qquad
q_2(x) =
\begin{cases}
1+\beta & x < 1/2,\\[2mm]
1-\beta & x \ge 1/2,
\end{cases}
\]
and note that both integrate to $1$ and that $\tv(q_1,q_2)=\beta$.
See \Cref{fig:q12} for illustration.

\begin{figure}[t]
\centering
\begin{tikzpicture}[x=6cm, y=1.2cm]

% y-levels
\def\yzero{0}
\def\ylo{0.5}   % 1 - beta
\def\yhi{1.0}   % 1 + beta

%%%%%%%%%%%%%%%%%%%%%%%%%%%%%%%%%%%%%%%%%%%%%%%%%%%%%%
% Top panel: q1
%%%%%%%%%%%%%%%%%%%%%%%%%%%%%%%%%%%%%%%%%%%%%%%%%%%%%%

\node at (0.5, \yhi+0.5) {$q_1$};

% Left half: density 1 - beta
\draw[fill=black!5] (0,0) rectangle (0.5,\ylo);
\draw (0,0) rectangle (0.5,\ylo);

% Right half: density 1 + beta
\draw[fill=black!20] (0.5,0) rectangle (1,\yhi);
\draw (0.5,0) rectangle (1,\yhi);

% x-axis
\draw[->] (0,-0.15) -- (1.05,-0.15) node[right]{\scriptsize $x$};
\draw (0,-0.15) node[below]{\scriptsize $0$};
\draw (0.5,-0.15) node[below]{\scriptsize $\tfrac12$};
\draw (1,-0.15) node[below]{\scriptsize $1$};

% y-axis ticks (schematic)
\draw (-0.03,\yzero) node[left]{\scriptsize $0$};
\draw (-0.03,\ylo)   node[left]{\scriptsize $1-\beta$};
\draw (-0.03,\yhi)   node[left]{\scriptsize $1+\beta$};

%%%%%%%%%%%%%%%%%%%%%%%%%%%%%%%%%%%%%%%%%%%%%%%%%%%%%%
% Bottom panel: q2
%%%%%%%%%%%%%%%%%%%%%%%%%%%%%%%%%%%%%%%%%%%%%%%%%%%%%%

\begin{scope}[yshift=-3cm]

\node at (0.5, \yhi+0.5) {$q_2$};

% Left half: density 1 + beta
\draw[fill=black!20] (0,0) rectangle (0.5,\yhi);
\draw (0,0) rectangle (0.5,\yhi);

% Right half: density 1 - beta
\draw[fill=black!5] (0.5,0) rectangle (1,\ylo);
\draw (0.5,0) rectangle (1,\ylo);

% x-axis
\draw[->] (0,-0.15) -- (1.05,-0.15) node[right]{\scriptsize $x$};
\draw (0,-0.15) node[below]{\scriptsize $0$};
\draw (0.5,-0.15) node[below]{\scriptsize $\tfrac12$};
\draw (1,-0.15) node[below]{\scriptsize $1$};

% y-axis ticks
\draw (-0.03,\yzero) node[left]{\scriptsize $0$};
\draw (-0.03,\ylo)   node[left]{\scriptsize $1-\beta$};
\draw (-0.03,\yhi)   node[left]{\scriptsize $1+\beta$};

\end{scope}
\end{tikzpicture}

\caption{Densities $q_1$ (top) and $q_2$ (bottom).  
The left half of $[0,1]$ has density $1-\beta$ in $q_1$ and $1+\beta$ in $q_2$, 
while the right half has the opposite pattern.}
\label{fig:q12}
\end{figure}

Next we introduce the families $\mathcal{D}_1,\mathcal{D}_2$.
Let $N$ be a large integer to be chosen later, and define
$k = \frac{1+\beta}{\beta}$. For simplicity (and without loss of generality), assume $k$ is an integer.
Partition $[0,1]$ into $2N$ intervals
\[
I_1,\dots,I_{2N},
\qquad |I_j|=\frac{1}{2N},
\]
and further partition each $I_j$ into $k$ equal subintervals
\[
I(j,1),\,I(j,2),\,\dots,\,I(j,k),
\qquad |I(j,\ell)|=\frac{1}{2Nk}.
\]

A distribution $p$ in either $\mathcal{D}_1$ or $\mathcal{D}_2$ is specified by choosing,
for each $j\in[2N]$, exactly one ``special'' small interval $I(j,\ell(j))$ with
$\ell(j)\in[k]$.
Different choices of the index function $\ell : [2N]\to[k]$ correspond to different
densities. (See \Cref{fig:D12} for Illustration.)

\smallskip
\noindent
\textbf{Definition of $\mathcal{D}_1$.}
Given a choice of $\ell(j)\in[k]$ for each $j\in[2N]$, the density $p$ is defined as follows.

\begin{itemize}
\item
For $j=1,\dots,N$ (first half):
\begin{align*}
p(x) &=
\begin{cases}
2, & x \in I(j,\ell(j)),\\[1mm]
1-\beta, & x \in I(j,\ell)\ \text{with }\ell\neq \ell(j).
\end{cases}
\end{align*}

\item
For $j=N+1,\dots,2N$ (second half):
\begin{align*}
p(x) &=
\begin{cases}
0, & x \in I(j,\ell(j)),\\[1mm]
1+\beta, & x \in I(j,\ell)\ \text{with }\ell\neq \ell(j).
\end{cases}
\end{align*}
\end{itemize}

Let $\mathcal{D}_1$ be the set of all densities obtained in this manner as $\ell$ ranges over
$[k]^{2N}$.

\smallskip
\noindent
\textbf{Definition of $\mathcal{D}_2$.}
This family is defined analogously but with the roles of ``high'' and ``low'' densities swapped
between the two halves.

For a choice of $\ell(j)\in[k]$, define $p$ by:

\begin{itemize}
\item
For $j=1,\dots,N$:
\begin{align*}
p(x) &=
\begin{cases}
0, & x \in I(j,\ell(j)),\\[1mm]
1+\beta, & x \in I(j,\ell)\ \text{with }\ell\neq \ell(j).
\end{cases}
\end{align*}

\item
For $j=N+1,\dots,2N$:
\begin{align*}
p(x) &=
\begin{cases}
2, & x \in I(j,\ell(j)),\\[1mm]
1-\beta, & x \in I(j,\ell)\ \text{with }\ell\neq \ell(j).
\end{cases}
\end{align*}
\end{itemize}

Let $\mathcal{D}_2$ be the set of all such densities.

%Finally, set
%\[
%q_1(x) =
%\begin{cases}
%1-\beta & x\leq 1/2\\
%1+\beta& x>1/2
%\end{cases}
%,~
%q_2(x) =
%\begin{cases}
%1+\beta & x\leq 1/2\\
%1-\beta& x>1/2,
%\end{cases}
%\]

%\begin{figure}
%\begin{center}
%\includegraphics[width=.75\textwidth]{Slide1.jpg}
%\end{center}
%\caption{An illustration of $q_1$ (left) and $q_2$ (right).}
%\label{fig:q12}
%\end{figure}

%In order to define $\D_1,\D_2$, pick a large integer $N=N(m)$ (to be determined later).
%Partition the unit interval into~$2N$ intervals $I_1,\ldots, I_{2N}$ of size $1/2N$ each.
%$\D_1$ is the family of distributions $p$ of the following form:
%let $R\subseteq [N]$ be a set of size $k=N\frac{\beta}{2(1+\beta)}$.
%Set (see \Cref{fig:D12} for illustration.)
%\[
%p(x) =
%\begin{cases}
%1-\beta & x\in  I_j,  j\notin R, j\leq N,\\
%2 & x\in  I_j,  j\in R, j\leq N,\\
%1+\beta & x\in  I_j,  j+N\notin R, j > N,\\
%0 & x\in  I_j,  j+N\in R, j\leq N.\\
%\end{cases}
%\]
%$\D_2$ is defined analogously as the family of distributions $p$ of the form
%\[
%p(x) =
%\begin{cases}
%1+\beta & x\in  I_j,  j\notin R, j\leq N,\\
%0 & x\in  I_j,  j\in R, j\leq N,\\
%1-\beta & x\in  I_j,  j+N\notin R, j > N,\\
%2 & x\in  I_j,  j+N\in R, j\leq N.\\
%\end{cases}
%\]

\begin{figure}[t]
\centering
\begin{tikzpicture}[x=0.5cm,y=1.2cm,scale=0.9]

% parameters for the schematic
\def\Nb{4}   % number of big intervals per half (schematic)
\def\kb{4}   % number of small intervals per big interval (schematic)
\def\BW{1.2} % width of a big interval
\pgfmathsetmacro{\SW}{\BW/\kb}  % small interval width

% y-levels for densities (schematic, only relative)
\def\yzero{0.0}
\def\ylo{0.5}   % 1 - beta
\def\yhi{1.0}   % 1 + beta
\def\yspike{1.6}% 2

%%%%%%%%%%%%%%%%%%%%%%%%%%%%%%%%%%%%%%%%%%%%%%%%%%%%
% Helper: draw one big interval with small intervals and a special one
%%%%%%%%%%%%%%%%%%%%%%%%%%%%%%%%%%%%%%%%%%%%%%%%%%%%

% #1 = left x of big interval
% #2 = baseline height
% #3 = special height
% #4 = label (for clarity, unused)
% #5 = choose special index in {0,...,kb-1}
\newcommand{\BigInterval}[5]{%
  \foreach \i in {0,...,\numexpr\kb-1} {%
    \pgfmathsetmacro{\xL}{#1 + \i*\SW}
    \pgfmathsetmacro{\xR}{#1 + (\i+1)*\SW}
    \ifnum\i=#5
      % special small interval
      \draw[fill=black!20] (\xL,0) rectangle (\xR,#3);
    \else
      % non-special small interval
      \draw[fill=black!5] (\xL,0) rectangle (\xR,#2);
    \fi
  }
  % outline of big interval
  \draw[thick] (#1,0) rectangle (#1+\BW,0);
}

%%%%%%%%%%%%%%%%%%%%%%%%%%%%%%%%%%%%%%%%%%%%%%%%%%%%
% Axis annotations (densities)
%%%%%%%%%%%%%%%%%%%%%%%%%%%%%%%%%%%%%%%%%%%%%%%%%%%%

% left y-axis for the top panel
\draw[->] (-1,0) -- (-1,\yspike+0.3);
\draw (-1,\yzero) node[left] {\scriptsize $0$};
\draw (-1,\ylo)   node[left] {\scriptsize $1-\beta$};
\draw (-1,\yhi)   node[left] {\scriptsize $1+\beta$};
\draw (-1,\yspike)node[left] {\scriptsize $2$};

%%%%%%%%%%%%%%%%%%%%%%%%%%%%%%%%%%%%%%%%%%%%%%%%%%%%
% Top panel: p in D1
%%%%%%%%%%%%%%%%%%%%%%%%%%%%%%%%%%%%%%%%%%%%%%%%%%%%

\node at (11, \yspike-1) {$p \in \mathcal{D}_1$};

% left half: j = 1,...,N, special interval height = 2, baseline = 1 - beta
\foreach \j in {0,...,\numexpr\Nb-1} {%
  \pgfmathsetmacro{\xstart}{\j*\BW}
  % choose some special index (schematically \j mod k)
  \pgfmathtruncatemacro{\spec}{mod(\j,\kb)}
  \BigInterval{\xstart}{\ylo}{\yspike}{}{ \spec }
}

% right half: j = N+1,...,2N, special interval height = 0, baseline = 1 + beta
\foreach \j in {0,...,\numexpr\Nb-1} {%
  \pgfmathsetmacro{\xstart}{(\Nb+\j)*\BW}
  \pgfmathtruncatemacro{\spec}{mod(\j,\kb)}
  % here baseline is 1 + beta, special is 0 (we indicate with white bar)
  \foreach \i in {0,...,\numexpr\kb-1} {%
    \pgfmathsetmacro{\xL}{\xstart + \i*\SW}
    \pgfmathsetmacro{\xR}{\xstart + (\i+1)*\SW}
    \ifnum\i=\spec
      \draw[fill=white] (\xL,0) rectangle (\xR,\yzero);
    \else
      \draw[fill=black!5] (\xL,0) rectangle (\xR,\yhi);
    \fi
  }
  \draw[thick] (\xstart,0) rectangle (\xstart+\BW,0);
}

% braces and labels for halves
\draw[decorate,decoration={brace,amplitude=4pt,mirror}]
      (0,-0.25) -- (\Nb*\BW,-0.25) node[midway,below=4pt] {\scriptsize left half of $[0,1]$};
\draw[decorate,decoration={brace,amplitude=4pt,mirror}]
      (\Nb*\BW,-0.25) -- (2*\Nb*\BW,-0.25) node[midway,below=4pt] {\scriptsize right half of $[0,1]$};

%%%%%%%%%%%%%%%%%%%%%%%%%%%%%%%%%%%%%%%%%%%%%%%%%%%%
% Bottom panel: p in D2 (shifted down)
%%%%%%%%%%%%%%%%%%%%%%%%%%%%%%%%%%%%%%%%%%%%%%%%%%%%

\begin{scope}[yshift=-3.2cm]

% axis
\draw[->] (-1,0) -- (-1,\yspike+0.3);
\draw (-1,\yzero) node[left] {\scriptsize $0$};
\draw (-1,\ylo)   node[left] {\scriptsize $1-\beta$};
\draw (-1,\yhi)   node[left] {\scriptsize $1+\beta$};
\draw (-1,\yspike)node[left] {\scriptsize $2$};

\node at (11, \yspike-0.75) {$p \in \mathcal{D}_2$};

% left half: special at 0, baseline 1 + beta
\foreach \j in {0,...,\numexpr\Nb-1} {%
  \pgfmathsetmacro{\xstart}{\j*\BW}
  \pgfmathtruncatemacro{\spec}{mod(\j,\kb)}
  \foreach \i in {0,...,\numexpr\kb-1} {%
    \pgfmathsetmacro{\xL}{\xstart + \i*\SW}
    \pgfmathsetmacro{\xR}{\xstart + (\i+1)*\SW}
    \ifnum\i=\spec
      \draw[fill=white] (\xL,0) rectangle (\xR,\yzero);
    \else
      \draw[fill=black!5] (\xL,0) rectangle (\xR,\yhi);
    \fi
  }
  \draw[thick] (\xstart,0) rectangle (\xstart+\BW,0);
}

% right half: special at 2, baseline 1 - beta
\foreach \j in {0,...,\numexpr\Nb-1} {%
  \pgfmathsetmacro{\xstart}{(\Nb+\j)*\BW}
  \pgfmathtruncatemacro{\spec}{mod(\j,\kb)}
  \foreach \i in {0,...,\numexpr\kb-1} {%
    \pgfmathsetmacro{\xL}{\xstart + \i*\SW}
    \pgfmathsetmacro{\xR}{\xstart + (\i+1)*\SW}
    \ifnum\i=\spec
      \draw[fill=black!20] (\xL,0) rectangle (\xR,\yspike);
    \else
      \draw[fill=black!5] (\xL,0) rectangle (\xR,\ylo);
    \fi
  }
  \draw[thick] (\xstart,0) rectangle (\xstart+\BW,0);
}

\draw[decorate,decoration={brace,amplitude=4pt,mirror}]
      (0,-0.25) -- (\Nb*\BW,-0.25) node[midway,below=4pt] {\scriptsize left half of $[0,1]$};
\draw[decorate,decoration={brace,amplitude=4pt,mirror}]
      (\Nb*\BW,-0.25) -- (2*\Nb*\BW,-0.25) node[midway,below=4pt] {\scriptsize right half of $[0,1]$};

\end{scope}

\end{tikzpicture}
\caption{Schematic illustration of the distributions in $\mathcal{D}_1$ (top) and $\mathcal{D}_2$ (bottom).
In each big interval (a block on the $x$-axis), the density is piecewise constant on $k$ equal small intervals.
For $p\in\mathcal{D}_1$, on the left half of $[0,1]$ each big interval has density $1-\beta$ except on one
small interval where the density is $2$, while on the right half it has density $1+\beta$ except on one
small interval where the density is $0$.  For $p\in\mathcal{D}_2$ the pattern is reversed.}
\label{fig:D12}
\end{figure}

%\begin{figure}
%\begin{center}
%\includegraphics[width=.75\textwidth]{Slide2.jpg}
%\end{center}
%\caption{An illustration of a distribution drawn from $\D_1$ (left) and of a distribution drawn from $\D_2$ (right).}
%\label{fig:D12}
%\end{figure}

The next claim, which follows from a trivial caclulation, yields the first two items.
\begin{claim}
Let $i,j\in\{0,1\}$ and let $p_j\in \D_j$. Then,
\[\tv(q_i,p_j) = \begin{cases}
\frac{1}{2}\beta  &i=j,\\
\frac{3}{2}\beta - \frac{2\beta^2}{1+\beta}> 3\cdot\frac{1}{2}\beta-2\beta^2 &i\neq j.
\end{cases}\]
%
%
%
%For every $p_1\in \D_1$:
%\[
%\tv(q_1,p_1)= \frac{1}{2}\beta  ~~\text{ and }~~  \tv(q_2,p_1)> 3\cdot\frac{1}{2}\beta-2\beta^2.
%\]

%\[
%\tv(q_1,p_1)=\frac{1}{2k}(1+\beta) = \frac{1}{2}\beta  ~~\text{ and }~~  \tv(q_2,p_1)= 3\cdot\frac{1}{2}\beta - \frac{2\beta^2}{1+\beta}> 3\cdot\frac{1}{2}\beta-2\beta^2.
%\]
%
%
%
%\[
%\tv(q_1,p_1)=\beta  ~~\text{ and }~~  \tv(q_2,p_1)= 3\beta - \frac{2\beta^2}{1+\beta}> 3\beta-2\beta^2 = 3(1-\beta)\tv(q_1,p_1) + \beta^2.
%\]
%Similarly,  for every $p_2\in \D_2$:
%\[
%\tv(q_2,p_2)= \frac{1}{2}\beta ~~\text{ and }~~ \tv(q_1,p_2) >  3\cdot\frac{1}{2}\beta-2\beta^2.\]
\end{claim}
%By setting $\eta=1-\frac{\alpha}{3}$ (equivalently $\alpha=3-3\eta$) it follows that $\beta \leq \eta^2$ (since $\beta\leq(1-\frac{\alpha}{3})^2$ by assumption);

%(less than $\frac{(3-\alpha)}{3}$ suffices).
The third item follows from the next claim.
\begin{claim}\label{clm:uniform-conditioning}
Let $\mathsf{D}\in\{\mathcal{D}_1^{\oplus m},\mathcal{D}_2^{\oplus m}\}$, and let
$E$ denote the event that every big interval $I_j$ contains at most one sample.
Then the conditional distribution of the $m$ samples under $\mathsf{D}$ coincides
with the conditional distribution of $m$ i.i.d.\ uniform samples on $[0,1]$
given~$E$; that is,
\[
  \mathsf{D} \mid E \;=\; U^m \mid E,
\]
where $U$ denotes the uniform distribution on $[0,1]$.
\end{claim}

\begin{proof}
Recall that a draw from $\mathcal{D}_1$ or $\mathcal{D}_2$ is generated by first
choosing, for each big interval $I_j$, a ``special'' small subinterval
$I(j,\ell(j))$, where $\ell(j)$ is drawn uniformly from $\{1,\ldots,k\}$,
independently across $j$, and then defining the density $p$ by assigning value
$2$ (or $0$) on $I(j,\ell(j))$ and value $1\pm\beta$ on the other small
subintervals inside $I_j$.  Thus, when $p$ is drawn uniformly from
$\mathcal{D}_1$ or $\mathcal{D}_2$, the special small interval in each big
interval is uniformly random.

By construction, for every $p\in\mathcal{D}_1\cup\mathcal{D}_2$ the measure
of each big interval $I_j$ is exactly $1/(2N)$.  Therefore, under~$\mathsf{D}$, 
the multiset of big intervals containing the $m$ samples has the
same distribution as under $U^m$.  Conditioning on the event $E$ (that no big
interval contains more than one sample) simply amounts to conditioning on the
event that these $m$ big intervals are distinct, so the set of occupied big
intervals is a uniformly random $m$-subset of $\{I_1,\dots,I_{2N}\}$ under both
$\mathsf{D}$ and $U^m$.

Now fix a big interval $I_j$.  Conditional on the event that a sample $X$ falls
in $I_j$, its location inside $I_j$ under $\mathsf{D}$ is determined as
follows: first a special small subinterval $I(j,\ell(j))$ is chosen uniformly
at random, then $X$ is drawn from the corresponding density, which assigns
value $2$ (or $0$) on $I(j,\ell(j))$ and $1\pm\beta$ on the other $k-1$ small
subintervals.  Averaging over the uniform choice of $\ell(j)$, these densities
average out to the constant value~$1$ on $I_j$, so $X$ is uniformly distributed
over $I_j$ given $X\in I_j$.  In particular, once the set of occupied big
intervals is fixed, under $\mathsf{D}$ we obtain one independent uniform point
in each of these intervals.

Putting the two steps together, a sample from $\mathsf{D}\mid E$ is generated
by (i) choosing a uniformly random $m$-subset of the big intervals and (ii)
drawing one uniform point from each chosen big interval.  This is exactly the
procedure that generates $U^m$ conditioned on~$E$.  Hence
$\mathsf{D}\mid E = U^m\mid E$, as claimed.
\end{proof}

The previous claim implies that 
\[
(\mathcal{D}_1^{\oplus m})\mid E
\;=\;
(\mathcal{D}_2^{\oplus m})\mid E,
\]
and therefore
\[
\tv\!\left((\mathcal{D}_1^{\oplus m})\mid E,\,
          (\mathcal{D}_2^{\oplus m})\mid E\right)=0.
\]
It remains to lower bound the probability of $E$ under both
$\mathcal{D}_1^{\oplus m}$ and $\mathcal{D}_2^{\oplus m}$.
Since each big interval $I_j$ has mass exactly $1/(2N)$ under every
$p\in\mathcal{D}_1\cup\mathcal{D}_2$, a sequence of $m$ i.i.d.\ samples
falls in $m$ distinct big intervals with probability
\[
  \Pr(E)
  \;=\;
  \Bigl(1-\frac{1}{2N}\Bigr)
  \Bigl(1-\frac{2}{2N}\Bigr)
  \cdots
  \Bigl(1-\frac{m-1}{2N}\Bigr)
  \;\approx\;
  \exp\!\Bigl(-\frac{m^2}{2N}\Bigr).
\]
Thus, by choosing $N$ sufficiently larger than $m^2$ (for instance,
$N = C m^2$ for a large enough constant $C$), the probability of~$E$
is at least $11/12$ under both meta-distributions.
Lemma~\ref{le:chain} then yields
\[
  \tv(\mathcal{D}_1^{\oplus m},\mathcal{D}_2^{\oplus m})
  \;\le\; \frac13.
\]

This completes the proof of Theorem~\ref{thm:lbdensities}.
The proof of Theorem~\ref{thm:lbfinite} is analogous.
In the finite-domain setting, we consider the domain
$[2N]\times[k]$ and identify each small interval $I(j,\ell)$ with the
pair $(j,\ell)$.  The distributions $q_1$ and $q_2$ are defined by
\[
q_1(j,\ell)=
\begin{cases}
\frac{1-\eta}{2Nk}, & j\le N,\\[1mm]
\frac{1+\eta}{2Nk}, & j>N,
\end{cases}
\qquad
q_2(j,\ell)=
\begin{cases}
\frac{1+\eta}{2Nk}, & j\le N,\\[1mm]
\frac{1-\eta}{2Nk}, & j>N,
\end{cases}
\]
and the families $\mathcal{D}_1$ and $\mathcal{D}_2$ are defined exactly
as in the continuous case by modifying, in each $j$, exactly one of the
$k$ points $(j,1),\dots,(j,k)$.  The same symmetry and collision
arguments apply verbatim.

\end{proof}

\bibliographystyle{abbrvnat}
\bibliography{distlearn}

\begin{thebibliography}{28}
\providecommand{\natexlab}[1]{#1}
\providecommand{\url}[1]{\texttt{#1}}
\expandafter\ifx\csname urlstyle\endcsname\relax
  \providecommand{\doi}[1]{doi: #1}\else
  \providecommand{\doi}{doi: \begingroup \urlstyle{rm}\Url}\fi

\bibitem[Aamand et~al.(2025)Aamand, Aliakbarpour, Chen, and
  Silwal]{aamand2025fastgoodhypothesis}
A.~Aamand, M.~Aliakbarpour, J.~Y. Chen, and S.~Silwal.
\newblock How fast can you find a good hypothesis?, 2025.
\newblock URL \url{https://arxiv.org/abs/2509.03734}.

\bibitem[Alon et~al.(2016)Alon, Moran, and Yehudayoff]{Alon16signrank}
N.~Alon, S.~Moran, and A.~Yehudayoff.
\newblock Sign rank versus {VC} dimension.
\newblock In V.~Feldman, A.~Rakhlin, and O.~Shamir, editors, \emph{Proceedings
  of the 29th Conference on Learning Theory, {COLT} 2016, New York, USA, June
  23-26, 2016}, volume~49 of \emph{{JMLR} Workshop and Conference Proceedings},
  pages 47--80. JMLR.org, 2016.

\bibitem[Ashtiani et~al.(2018{\natexlab{a}})Ashtiani, Ben-David, Harvey, Liaw,
  Mehrabian, and Plan]{Ashtiani18mixtures}
H.~Ashtiani, S.~Ben-David, N.~Harvey, C.~Liaw, A.~Mehrabian, and Y.~Plan.
\newblock Nearly tight sample complexity bounds for learning mixtures of
  gaussians via sample compression schemes.
\newblock In S.~Bengio, H.~Wallach, H.~Larochelle, K.~Grauman, N.~Cesa-Bianchi,
  and R.~Garnett, editors, \emph{Advances in Neural Information Processing
  Systems 31}, pages 3416--3425. Curran Associates, Inc., 2018{\natexlab{a}}.

\bibitem[Ashtiani et~al.(2018{\natexlab{b}})Ashtiani, Ben{-}David, and
  Mehrabian]{Ashtiani18mixtures1}
H.~Ashtiani, S.~Ben{-}David, and A.~Mehrabian.
\newblock Sample-efficient learning of mixtures.
\newblock In \emph{Proceedings of the Thirty-Second {AAAI} Conference on
  Artificial Intelligence, (AAAI-18), the 30th innovative Applications of
  Artificial Intelligence (IAAI-18), and the 8th {AAAI} Symposium on
  Educational Advances in Artificial Intelligence (EAAI-18), New Orleans,
  Louisiana, USA, February 2-7, 2018}, pages 2679--2686, 2018{\natexlab{b}}.

\bibitem[Bassily et~al.(2016)Bassily, Nissim, Smith, Steinke, Stemmer, and
  Ullman]{Bassily16stability}
R.~Bassily, K.~Nissim, A.~Smith, T.~Steinke, U.~Stemmer, and J.~Ullman.
\newblock Algorithmic stability for adaptive data analysis.
\newblock In \emph{Proceedings of the Forty-eighth Annual ACM Symposium on
  Theory of Computing}, STOC '16, pages 1046--1059, New York, NY, USA, 2016.
  ACM.
\newblock ISBN 978-1-4503-4132-5.
\newblock \doi{10.1145/2897518.2897566}.

\bibitem[Chan et~al.(2014)Chan, Diakonikolas, Servedio, and
  Sun]{Chan14histograms}
S.~Chan, I.~Diakonikolas, R.~A. Servedio, and X.~Sun.
\newblock Near-optimal density estimation in near-linear time using
  variable-width histograms.
\newblock In \emph{{NIPS}}, pages 1844--1852, 2014.

\bibitem[Daniely and Shalev{-}Shwartz(2014)]{Daniely14proper}
A.~Daniely and S.~Shalev{-}Shwartz.
\newblock Optimal learners for multiclass problems.
\newblock In M.~Balcan, V.~Feldman, and C.~Szepesv{\'{a}}ri, editors,
  \emph{Proceedings of The 27th Conference on Learning Theory, {COLT} 2014,
  Barcelona, Spain, June 13-15, 2014}, volume~35 of \emph{{JMLR} Workshop and
  Conference Proceedings}, pages 287--316. JMLR.org, 2014.

\bibitem[Devroye and Gyorfi(1985)]{devroye85nonparametric}
L.~Devroye and L.~Gyorfi.
\newblock \emph{Nonparametric Density Estimation: The L1 View}.
\newblock Wiley Interscience Series in Discrete Mathematics. Wiley, 1985.
\newblock ISBN 9780471816461.

\bibitem[Devroye and Lugosi(2001)]{Devroye01combinatorial}
L.~Devroye and G.~Lugosi.
\newblock \emph{{Combinatorial methods in density estimation}}.
\newblock Springer, 2001.

\bibitem[Devroye and Lugosi(2004)]{Devroye04bin}
L.~Devroye and G.~Lugosi.
\newblock Bin width selection in multivariate histograms by the combinatorial
  method.
\newblock \emph{Test}, 13\penalty0 (1):\penalty0 129--145, Jun 2004.
\newblock ISSN 1863-8260.
\newblock \doi{10.1007/BF02603004}.

\bibitem[Diakonikolas(2016)]{Diakonikolas16survey}
I.~Diakonikolas.
\newblock Learning structured distributions.
\newblock In \emph{Handbook of Big Data}, pages 267--283. Chapman and Hall/CRC,
  2016.

\bibitem[Diakonikolas et~al.(2017)Diakonikolas, Kane, and
  Stewart]{Diakonikolas17mixtureslower}
I.~Diakonikolas, D.~M. Kane, and A.~Stewart.
\newblock Statistical query lower bounds for robust estimation of
  high-dimensional gaussians and gaussian mixtures.
\newblock In \emph{{FOCS}}, pages 73--84. {IEEE} Computer Society, 2017.

\bibitem[Diakonikolas et~al.(2018{\natexlab{a}})Diakonikolas, Kane, and
  Stewart]{Diakonikolas18mixtures}
I.~Diakonikolas, D.~M. Kane, and A.~Stewart.
\newblock List-decodable robust mean estimation and learning mixtures of
  spherical gaussians.
\newblock In \emph{{STOC}}, pages 1047--1060. {ACM}, 2018{\natexlab{a}}.

\bibitem[Diakonikolas et~al.(2018{\natexlab{b}})Diakonikolas, Li, and
  Schmidt]{Diakonikolas18histograms}
I.~Diakonikolas, J.~Li, and L.~Schmidt.
\newblock Fast and sample near-optimal algorithms for learning multidimensional
  histograms.
\newblock In S.~Bubeck, V.~Perchet, and P.~Rigollet, editors, \emph{Proceedings
  of the 31st Conference On Learning Theory}, volume~75 of \emph{Proceedings of
  Machine Learning Research}, pages 819--842. PMLR, 06--09 Jul
  2018{\natexlab{b}}.

\bibitem[Dwork et~al.(2015)Dwork, Feldman, Hardt, Pitassi, Reingold, and
  Roth]{Dwork15adaptive}
C.~Dwork, V.~Feldman, M.~Hardt, T.~Pitassi, O.~Reingold, and A.~Roth.
\newblock The reusable holdout: Preserving validity in adaptive data analysis.
\newblock \emph{Science}, 349\penalty0 (6248):\penalty0 636--638, 2015.
\newblock ISSN 0036-8075.
\newblock \doi{10.1126/science.aaa9375}.

\bibitem[Jiao et~al.(2018)Jiao, Han, and Weissman]{Jiao18minimax}
J.~Jiao, Y.~Han, and T.~Weissman.
\newblock Minimax estimation of the $l_{1}$ distance.
\newblock \emph{IEEE Transactions on Information Theory}, 64\penalty0
  (10):\penalty0 6672--6706, Oct 2018.
\newblock ISSN 0018-9448.
\newblock \doi{10.1109/TIT.2018.2846245}.

\bibitem[Kalai et~al.(2012)Kalai, Moitra, and Valiant]{Kalai12mixtures}
A.~T. Kalai, A.~Moitra, and G.~Valiant.
\newblock Disentangling gaussians.
\newblock \emph{Commun. {ACM}}, 55\penalty0 (2):\penalty0 113--120, 2012.
\newblock \doi{10.1145/2076450.2076474}.

\bibitem[Kothari et~al.(2018)Kothari, Steinhardt, and Steurer]{Kothari18robust}
P.~K. Kothari, J.~Steinhardt, and D.~Steurer.
\newblock Robust moment estimation and improved clustering via sum of squares.
\newblock In \emph{Proceedings of the 50th Annual {ACM} {SIGACT} Symposium on
  Theory of Computing, {STOC} 2018, Los Angeles, CA, USA, June 25-29, 2018},
  pages 1035--1046, 2018.
\newblock \doi{10.1145/3188745.3188970}.

\bibitem[Lugosi and Nobel(1996)]{lugosi96histograms}
G.~Lugosi and A.~Nobel.
\newblock Consistency of data-driven histogram methods for density estimation
  and classification.
\newblock \emph{Ann. Statist.}, 24\penalty0 (2):\penalty0 687--706, 04 1996.
\newblock \doi{10.1214/aos/1032894460}.

\bibitem[Mahalanabis and Stefankovic(2008)]{Mahalanabis08density}
S.~Mahalanabis and D.~Stefankovic.
\newblock Density estimation in linear time.
\newblock In R.~A. Servedio and T.~Zhang, editors, \emph{21st Annual Conference
  on Learning Theory - {COLT} 2008, Helsinki, Finland, July 9-12, 2008}, pages
  503--512. Omnipress, 2008.

\bibitem[M{\"u}ller(1997)]{muller1997integral}
A.~M{\"u}ller.
\newblock Integral probability metrics and their generating classes of
  functions.
\newblock \emph{Advances in Applied Probability}, 29\penalty0 (2):\penalty0
  429--443, 1997.

\bibitem[Pearson(1895)]{Pearson95contributions}
K.~Pearson.
\newblock Contributions to the mathematical theory of evolution. ii. skew
  variation in homogeneous material.
\newblock \emph{Philosophical Trans. of the Royal Society of London},
  186:\penalty0 343--414, 1895.

\bibitem[{Sauer}(1972)]{sauer72sauer}
N.~{Sauer}.
\newblock {On the density of families of sets.}
\newblock \emph{{J. Comb. Theory, Ser. A}}, 13:\penalty0 145--147, 1972.
\newblock ISSN 0097-3165.
\newblock \doi{10.1016/0097-3165(72)90019-2}.

\bibitem[Schapire and Freund(2012)]{schapire2012boosting}
R.~E. Schapire and Y.~Freund.
\newblock \emph{Boosting: Foundations and algorithms}.
\newblock MIT press, 2012.

\bibitem[{Vapnik} and {Chervonenkis}(1971)]{Vapnik71uniform}
V.~{Vapnik} and A.~{Chervonenkis}.
\newblock {On the uniform convergence of relative frequencies of events to
  their probabilities.}
\newblock \emph{{Theory Probab. Appl.}}, 16:\penalty0 264--280, 1971.
\newblock ISSN 0040-585X; 1095-7219/e.
\newblock \doi{10.1137/1116025}.

\bibitem[von Neumann(1928)]{Neumann1928}
J.~von Neumann.
\newblock Zur theorie der gesellschaftsspiele.
\newblock \emph{Mathematische Annalen}, 100:\penalty0 295--320, 1928.

\bibitem[Yatracos(1985)]{yatracos85}
Y.~G. Yatracos.
\newblock Rates of convergence of minimum distance estimators and kolmogorov's
  entropy.
\newblock \emph{Ann. Statist.}, 13\penalty0 (2):\penalty0 768--774, 06 1985.
\newblock \doi{10.1214/aos/1176349553}.

\bibitem[Yu(1997)]{yu1997assouad}
B.~Yu.
\newblock {A}ssouad, {F}ano, and {L}e {C}am.
\newblock In \emph{Festschrift for Lucien Le Cam}, pages 423--435. Springer,
  1997.

\end{thebibliography}

%The bound $\tilde O(\frac{\sqrt{\log\lvert X\rvert}\log n}{\eps^{7/2}})$ is efficient
%in the sense that its polynomial in the input length (i.e.\ polynomial in $\log n$ and $\log\lvert X\rvert$).
%The bound $O(\frac{\sqrt{n}}{\eps^{5/2}}\log(1/\delta))$ is better when $\lvert X\rvert$ is very large (possibly infinite).

\newpage

\appendix

\section{Proof of \Cref{lem:polyt}}\label{sec:lempolyt}
\begin{proof}
The desired equality hinges on the Minimax Theorem:
\begin{align*}
\min_{v\in\Q_{\tv}} h\cdot v =\min_{v\in\Q_{\tv}} \sum_{i}h_iv_i &= \min_{p\in \Delta(\X)}\sum_{i}h_i\tv(p,q_i)\\
					  &= \min_{p\in \Delta(\X)}\max_{f_i:X\to[0,1]}\sum_i h_i\bigl(\Ex_{p}[f_i] - \Ex_{q_i}[f_i]\bigr)\\
					  &= \max_{f_i:\X\to[0,1]}\min_{p\in \Delta(\X)}\sum_i h_i\bigl(\Ex_{p}[f_i] - \Ex_{q_i}[f_i]\bigr)\tag{by the Minimax Theorem~\cite{Neumann1928}}\\
					  &=\max_{f_i\in\conv(\F_i)}\min_{p\in \Delta(\X)}\sum_i h_i\bigl(\Ex_{p}[f_i] - \Ex_{q_i}[f_i]\bigr) \tag{this is the technical part that is derived below}\\
					  &=\min_{p\in \Delta(\X)}\max_{f_i\in\conv(\F_i)}\sum_i h_i\bigl(\Ex_{p}[f_i] - \Ex_{q_i}[f_i]\bigr)\tag{by the Minimax Theorem}\\
					  &=\min_{p\in \Delta(\X)}\max_{f_i\in\F_i}\sum_i h_i\bigl(\Ex_{p}[f_i] - \Ex_{q_i}[f_i]\bigr)\tag{a linear function over a convex set is maximized at a vertex}\\
					  &= \min_{p\in \Delta(\X)}\sum_{i}h_id_{\F_i}(p,q_i)\leq h_id_{\F}(p,q_i) = \min_{v\in\Q_{\F}} h\cdot v.
\end{align*}
%where the last inequality follows because the maximum of a linear function over a convex set (i.e.\ $\conv(\F)^n$)
%is attained by a vertex of the convex set (i.e.\ by some element of $\F^n$, which is a list of functions $f_i$ in $\F$).

We next turn to prove the main inequality:
\[ \max_{f_i:\X\to[0,1]}\min_{p\in \Delta(\X)}\sum_i h_i\bigl(\Ex_{p}[f_i] - \Ex_{q_i}[f_i]\bigr) = \max_{f_i\in\conv(\F_i)}\min_{p\in \Delta(\X)}\sum_i h_i\bigl(\Ex_{p}[f_i] - \Ex_{q_i}[f_i]\bigr).\]
First, note that the direction ``$\geq$'' is trivial since in the left-hand-side the maximum is not restricted to $f_i\in \conv(F)$.
The other direction follows by analyzing the $f_i$'s that maximize the program
\begin{equation}\label{eq:program}
\max_{f_i:\X\to[0,1]}\min_{p\in \Delta(\X)}\sum_i h_i\bigl(\Ex_{p}[f_i] - \Ex_{q_i}[f_i]\bigr).
\end{equation}
Let us first write the objective $T(f_i)=T(f_1,\ldots,f_n):=\min_{p\in \Delta(\X)}\sum_i h_i(\Ex_{p}[f_i] - \Ex_{q_i}[f_i])$ more explicitly:
\begin{align*}
T(f_i)&=\min_{p\in \Delta(\X)} \sum_i h_i\bigl(\Ex_{p}[f_i] - \Ex_{q_i}[f_i]\bigr) \\
        &=  \min_{p\in  \Delta(\X)}\Bigl(\sum_x p(x)\sum_i h_if_i(x) - \sum_x\sum_i q_i(x)h_if_i(x)\Bigr) \\
        &= \min_{p\in \Delta(\X)}\Bigl(\sum_x p(x)\sum_i h_if_i(x)\Bigr) - \sum_x\sum_i q_i(x)h_if_i(x).
\end{align*}
We want to show that there exists a maximizer $f_i^*$ of $T(f_i)$ such that $f_i^*\in \conv(\F_i)$.
To see this, it will be more convenient to express $T(f_i)$ in the following maximization form:
\begin{claim}\label{c:lp}
For every choice of the $f_i$'s the function $T(f_i)$ equals to the value of the following linear program in the variable $\lambda\in\R$:
\begin{align}
&\max_{\lambda}~~~ \lambda - \sum_x\sum_i h_i f_i(x) q_i(x)  \nonumber\\
&\text{subject to}~~~  \lambda \leq \sum_i h_if_i(x),~~~ \forall x\in \X \nonumber.
\end{align}
\end{claim}
\begin{proof}
We show that both $T(f_i)$ and the value of the above program are equal to
\[ \min_x \bigl(h_if_i(x)\bigr) - \sum_x\sum_i h_i f_i(x) q_i(x).\]
Indeed, for the linear program it follows directly from its definition.

To derive it also for $T(f_i)$,
recall that we already established that
\[T(f_i) =\min_{p\in \Delta(\X)}\Bigl(\sum_x p(x)\sum_i h_if_i(x)\Bigr) - \sum_x\sum_i q_i(x)h_if_i(x), \]
Thus, its value is obtained by distributions $p^*$
that minimize $\sum_x p(x)\sum_i h_if_i(x)$.
Clearly, $p^*$ minimizes this sum if it concentrates all its weight
on the $x$'s that minimizes $\sum_i h_if_i(x)$,
and therefore $T(f_i) = \min_x h_if_i(x) - \sum_x\sum_i q_i(x)h_if_i(x)$,
as required.
\end{proof}

By Claim~\ref{c:lp} it suffices to show that there are $f_i^*\in\conv(\F)$ that maximize the following linear program
\begin{align}
&\max_{f_i, \lambda}~~~ \lambda - \sum_x\sum_i q_i(x)h_i f_i(x)  \label{eq:program}\\
&\text{subject to}~~~  \lambda \leq \sum_i h_if_i(x),~~~ \forall x\in \X \nonumber\\
&\text{and to}~~~ f_i:\X\to[0,1],~~~ \forall i\leq n. \nonumber
\end{align}
Note that since the maximization is over both $\lambda$ and the $f_i$'s then
we can first maximize over the $f_i$'s (keeping $\lambda$ fixed), and then optimize over $\lambda$.
In other words, it suffices to show that for a fixed $\lambda$, the optimal $f_i$'s satisfy $f_i\in \conv(\F_i)$.
Since $\lambda$ is fixed, we can consider the simpler objective of
\begin{align*}
\min_{f_i}~~~ \sum_x\sum_i q_i(x)h_i f_i(x) .
\end{align*}
Since the constraints over different $x$'s are independent,
we can optimize each $f_i(x)$ point-wise by solving:
\begin{align*}
&\min_{f_i}~~~\sum_i q_i(x)h_i f_i(x) \\
&\text{subject to}~~~  \sum_i h_if_i(x) \geq \lambda , \nonumber\\
&\text{and to}~~~ f_i(x)\in[0,1],~~~ \forall i\leq n. \nonumber
\end{align*}
The latter form is easier to handle:
Sort the $q_i(x)$ according to their values;
for simplicity and without loss of generality
assume that $q_1(x)\leq q_2(x)\leq\ldots q_n(x)$.
We claim that an optimal solution can be obtained
by traversing the $i$ from $1$ to $n$,
and setting the corresponding $f_i$ to as large as possible
until feasibility is achieved (i.e.\ until $\sum_i h_if_i(x) = \lambda$).
More formally, the following solution is optimal:
\[
f_i(x) =
\begin{cases}
1 &\sum_{j \leq  i} h_j < \lambda, \\
\frac{\lambda -  \sum_{j < i} h_j}{h_i} &\sum_{j <  i} h_j < \lambda \text{ and } \sum_{j < i} h_j  + h_i\geq \lambda,\\
0 &\text{otherwise.}
\end{cases}
\]
Indeed, else there would be some $i$ with $\sum_{j <  i} h_j > \lambda$ for which $q_i(x) > 0$,
and we could decrease $f_i(x)$ to $0$ and increase $f_j(x)$ for for some $j$'s with $j<i$,
which could only improve (decrease) the objective.

The proof is finished by noticing that
\[f_i(x) = 1_{\sum_{j\neq i} h_jA_{i,j}(x) < \lambda } ~~~~\text{ or } ~~~~ f_i(x) = t\cdot1_{\sum_{j\neq i} h_jA_{i,j}(x) < \lambda},\]
for $t= \frac{\lambda -  \sum_{j < i} h_j}{h_i} \leq 1$,
and in either way $f_i\in\conv(\F_i)$ (note that indeed $t\cdot 1_{\sum_j h_jA_{i,j}(x) \leq \lambda^*}$ is in $\conv(\F_i)$
since it is a convex combination of $1_{\sum_j h_jA_{i,j}(x) \leq \lambda^*}$ and the all-zeros function,
which are both in $\F_i$).

\end{proof}

\end{document}